\documentclass[twoside]{article}

\usepackage[accepted]{aistats2026}
\usepackage[round]{natbib}

\usepackage{amsmath}
\usepackage{amsfonts}
\usepackage{amssymb}
\usepackage{pifont}
\usepackage{bm}
\usepackage{amsthm}
\usepackage{booktabs}
\usepackage{adjustbox}
\usepackage{url}
\usepackage{caption}
\usepackage{hyperref}

\newcommand{\cmark}{\ding{51}}
\newcommand{\xmark}{\ding{55}}

\usepackage{algorithmic}
\usepackage{algorithm}

\theoremstyle{plain}
\newtheorem{theorem}{Theorem}[section]

\newtheorem{lemma}[theorem]{Lemma}

\theoremstyle{definition}

\theoremstyle{remark}

\begin{document}

%
\runningtitle{Injecting Measurement Information Yields a Fast and Robust Inverse Problem Solver}
%
\runningauthor{Patsenker\textsuperscript{*}, Li\textsuperscript{*}, Ko, Jia, Kluger}

\twocolumn[

\aistatstitle{Injecting Measurement Information Yields a Fast and Noise-Robust Diffusion-Based Inverse Problem Solver}

\aistatsauthor{ Jonathan Patsenker\textsuperscript{*} \And Henry Li\textsuperscript{*} \And  Myeongseob Ko}

\aistatsaddress{ Yale University \And  Yale University \And Virginia Tech}

\aistatsauthor{Ruoxi Jia \And Yuval Kluger}

\aistatsaddress{Virginia Tech \And Yale University}

]

\begin{abstract}
Diffusion models have been firmly established as principled zero-shot solvers for linear and nonlinear inverse problems, owing to their powerful image prior and iterative sampling algorithm. These approaches often rely on Tweedie's formula, which relates the diffusion variate $\mathbf{x}_t$ to the posterior mean $\mathbb{E} [\mathbf{x}_0 | \mathbf{x}_t]$, in order to guide the diffusion trajectory with an estimate of the final denoised sample $\mathbf{x}_0$.
However, this does not consider information from the measurement $\mathbf{y}$, which must then be integrated downstream. In this work, we propose to estimate the conditional posterior mean $\mathbb{E} [\mathbf{x}_0 | \mathbf{x}_t, \mathbf{y}]$, which can be formulated as the solution to a lightweight, single-parameter maximum likelihood estimation problem.
The resulting prediction can be integrated into any standard sampler, resulting in a fast and memory-efficient inverse solver. Our optimizer is amenable to a noise-aware likelihood-based stopping criteria that is robust to measurement noise in $\mathbf{y}$.
We demonstrate comparable or improved performance against a wide selection of contemporary inverse solvers across multiple datasets and tasks\footnote{For additional visualizations and results, please visit the \href{https://diffusion-conditional-sampling.github.io/}{project website}.}.
\end{abstract}

\begin{figure}[h!]
    \centering
    \includegraphics[width=0.45\textwidth]{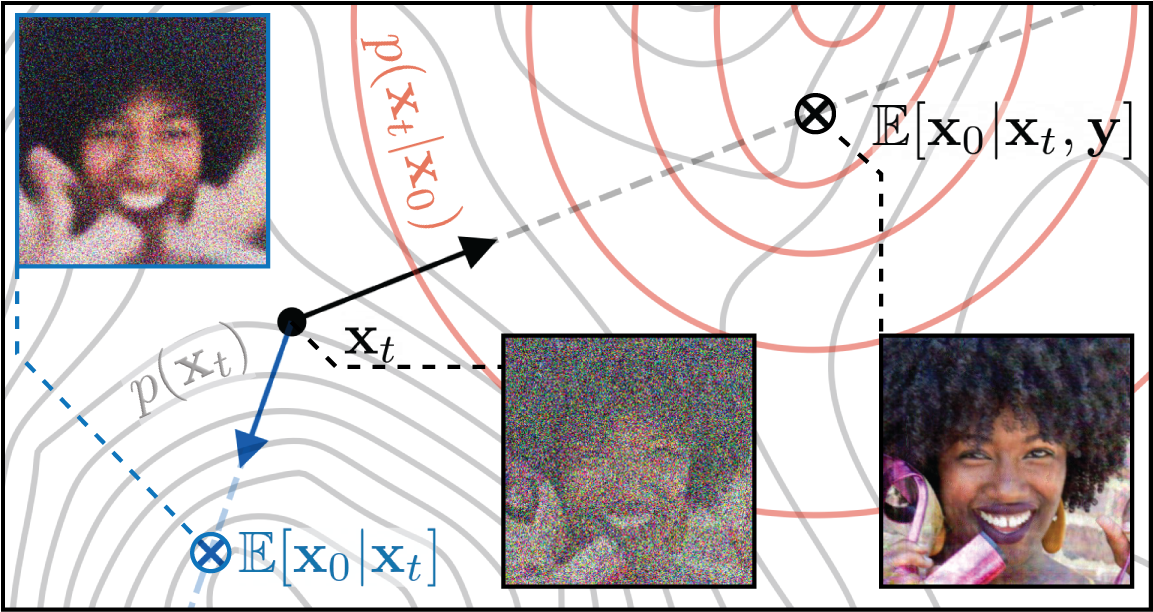}
    \caption{
    The posterior mean before and after conditioning on $\mathbf{y}$ for an inpainting inverse problem.
    }
    \label{fig:tweedies_folly_viz} 
\end{figure}

\section{INTRODUCTION}

In this work, we study a broad class of problems involving the recovery of a signal $\mathbf{x}$ from a measurement
\begin{equation}
    \mathbf{y} = \mathcal{A}(\mathbf{x}) + \boldsymbol\eta.
    \label{eq:intro_inverse_problem}
\end{equation}

with noise $\boldsymbol\eta$ and measurement operator $\mathcal{A}$. Known as inverse problems, such formulations appear in a multitude of fields, with applications including acoustic reconstruction \citep{kac1966can}, seismic profiling \citep{hardage1985vertical}, X-ray computed tomography and magnetic resonance imaging \citep{suetens2017fundamentals}, and a large number of computer vision reconstruction tasks such as inpainting, deconvolution, colorization, super-resolution, and phase retrieval \citep{andrews1977digital}.

\begin{figure*}
    \begin{minipage}[t]{.64\linewidth}
        \includegraphics[width=\linewidth]{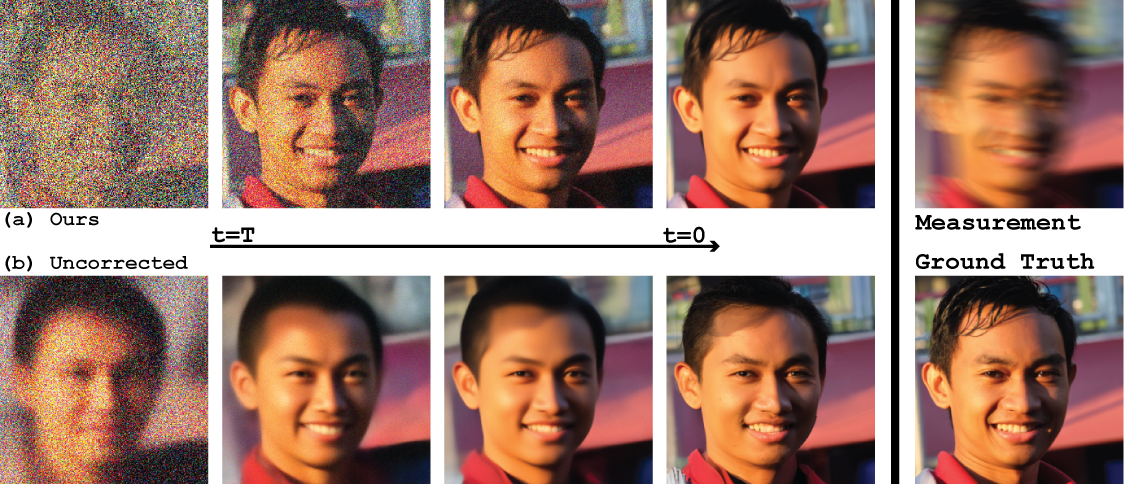}
        \caption{(a) $\mathbb{E}[\mathbf{x}_0|\mathbf{x}_t]$ via the unconditional score versus (b) $\mathbb{E}[\mathbf{x}_0|\mathbf{x}_t, \mathbf{y}]$ via the forward process score estimate obtained by our likelihood maximizer for an image in the FFHQ $256 \times 256$ dataset with motion blur applied. With the unconditional score, $\hat{\mathbf{x}}_0$ estimates the posterior mean of the dataset, rather than a sample $\mathbf{x}$ that satisfies $\mathcal{A}(\mathbf{x}) \approx \mathbf{y}$, especially at $T \gg 0$ (Section \ref{sec:tweedies_explained}).}
        \label{fig:corrected_uncorrected_example}
    \end{minipage}
    \hfill
    \begin{minipage}[t]{.32\linewidth}
        \centering
        \includegraphics[width=.82\linewidth]{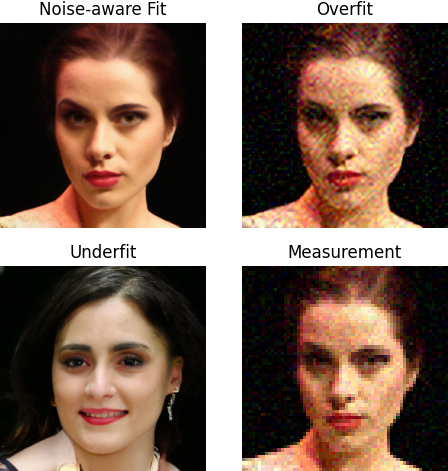}
        \caption{The hazard of over- or under-fitting for a super-resolution task. An ideal noise-aware fit balances between the prior and the noisy measurement $\mathbf{y}$.}
        \label{fig:overfit_underfit_example}
    \end{minipage}
\end{figure*}

\begin{figure}[h]
    \fontsize{9pt}{9pt}\selectfont
    \centering
    \resizebox{\linewidth}{!}{
        \begin{tabular}{l|lccc}
            \textbf{Solver} & \textbf{Type} & \begin{tabular}{@{}c@{}}\textbf{No NFE}  \\ \textbf{Backprop} \end{tabular}  & \textbf{Runtime} & \textbf{Memory}\\
            \hline
            \textbf{DCS} (Ours) & Hybrid & \cmark & \textbf{1x} & \textbf{1x}\\
            \midrule
            MCG &  Projection & \xmark &$2.6\mathbf{x}$ & $3.2\mathbf{x}$ \\
            DPS &  Posterior & \xmark &$2.5\mathbf{x}$ & $3.2\mathbf{x}$ \\
            DPS-JF & Posterior & \cmark & $1.2\mathbf{x}$ & $1.1\mathbf{x}$\\
            DDNM & Projection & \cmark &$1.5\mathbf{x}$ & $1\mathbf{x}$ \\
            RED-Diff & Projection & \cmark & 1.5$\mathbf{x}$ & 1$\mathbf{x}$ \\
            LGD-MC & Posterior & \xmark & 2$\mathbf{x}$ & $3.2\mathbf{x}$\\
            \bottomrule
        \end{tabular}
    }
    \captionof{table}{Overview of pixel-based solvers used for comparisons in this work. We list the type (Section \ref{sec:tweedies}), whether it requires backpropagation through a neural function evaluation, runtime, and memory footprint.}
    \label{table:inverse_solvers_overview}
\end{figure}

Often, inverting $\mathcal{A}$ is numerically intractable (Appendix \ref{sec:noninvertible_a}), meaning that solutions $\mathbf{x}$ satisfying $\mathcal{A}(\mathbf{x}) = \mathbf{y}$ are not directly obtainable or unique \citep{vogel2002computational}. Moreover, due to measurement noise, it is often possible, but not desirable to fit perfectly to $\mathbf{y}$ for risk of overfitting to $\boldsymbol \eta$ \citep{aster2018parameter}. Therefore, a fundamental step in solving inverse problems is deciding how to select the best option from an equivalence class of solutions, i.e., choosing $\mathbf{x}_* \in \{ \mathbf{x} : \mathcal{A}(\mathbf{x}) \approx \mathbf{y} \}$.

In classical solvers, this is carried out by a regularizer on a normed error loss \citep{engl1996regularization}:
\begin{equation}
    \mathbf{x}_* = \underset{\mathbf{x}}{\arg\min} \; R(\mathbf{x}) \hspace{.25in} \text{s.t.} \hspace{.1in} ||\mathcal{A}(\mathbf{x}) - \mathbf{y}|| \leq \epsilon,
    \label{eq:classical_approaches}
\end{equation}
where $\epsilon$ is a soft error margin and $R$ is a simple function that satisfies user-specified heuristics, e.g., smoothness or total variation \citep{beck2009fast}. However, such approaches often fail to produce realistic results, as it can be difficult to design $R$ without introducing strong biases. With the advent of deep generative models, practitioners found that restricting solutions to the range of a generative model $G$ can greatly improve realism: let $\mathbf{x} = G(\mathbf{w})$ and optimize over $\mathbf{w}$, which can be latent inputs \citep{bora2017compressed} or weights \citep{ulyanov2018deep} of a deep neural network. Overall, these methods improve the fidelity of $\mathbf{x}$, but they lack interpretability and require a judiciously selected $R$ and $\epsilon$.

\begin{figure*}[h!]
    \centering
    \includegraphics[width=.95\linewidth]{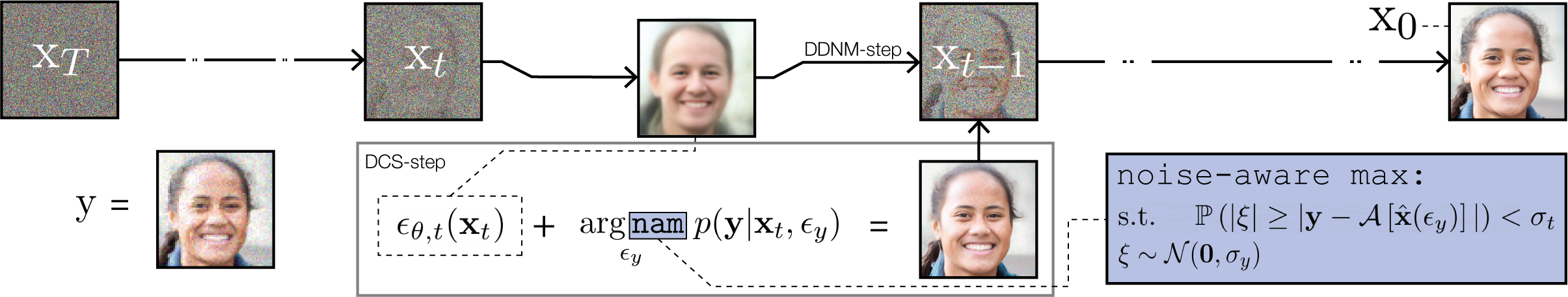}
    \caption{An illustration of our proposed sampling algorithm. An initial noise prediction $\epsilon_\theta$ is corrected by the solution $\epsilon_\mathbf{y}$ of a \textbf{n}oise-\textbf{a}ware \textbf{m}aximization scheme of the measurement likelihood $p(\mathbf{y} | \mathbf{x}_t, \epsilon_y)$. This results in the corrected forward process noise prediction $(\epsilon_\theta + \epsilon_y) \approx -\sigma_t^{-1} \nabla \log p_t(\mathbf{x}_t | \mathbf{x}_0)$. For details see Section \ref{sec:ours}.}
    \label{fig:main}
\end{figure*}

Recently, great strides have been made in solving inverse problems with diffusion models \citep{ho2020denoising}, which produce diverse, realistic samples \citep{dhariwal2021diffusion,esser2024scaling} with robust generalization guarantees \citep{kadkhodaie2023generalization}. Moreover, they are interpretable, directly modeling the (Stein) score $\nabla \log p_t$. Unconditional sampling proceeds by reversing a noising process on $\mathbf{x}_0 \sim p_{\text{data}}$. Solvers then employ a conditional sampling process via a \texttt{guidance} term that pushes samples toward solutions consistent with $\mathbf{y}$:
\begin{equation*}
	\mathbf{x}_{t-1} =  \overbrace{\texttt{denoise}[\mathbf{x}_t, \nabla \log p_t(\mathbf{x}_t)]}^{\texttt{unconditional sampling}} \; + \; \texttt{guidance}
\end{equation*}
 This approach faces a fundamental challenge: the \texttt{guidance} term depends on a consistency error $||\mathcal{A}(\mathbf{x}) - \mathbf{y}||$ that is only tractable for $\mathbf{x} = \mathbf{x}_0$ \citep{chung2022diffusion}. Such methods rely (explicitly or implicitly via \citep{song2020denoising}) on Tweedie's formula, which estimates $\mathbf{x}_0$ given a noise prediction $\bm{\epsilon}_t \approx -\sigma_t \nabla \log p_t(\mathbf{x}_t)$:
\begin{equation}
	\hat{\mathbf{x}}_0 = \mathbb{E} \left [ \mathbf{x}_0 | \mathbf{x}_t \right ]= \frac{1}{\sqrt{\alpha_t}}\left(\mathbf{x}_t - \sigma_t \bm{\epsilon}_t\right).
	\label{eq:tweedies_intro}
\end{equation}
This approximation is then substituted for $\mathbf{x}$ in the consistency error, producing a differentiable function with respect to $\mathbf{x}_t$.

This naive implementation of Eq. \ref{eq:tweedies_intro} introduces significant approximation error, as $\mathbb{E}[\mathbf{x}_0 | \mathbf{x}_t] = \mathbf{x}_0$ if and only if the marginal distribution of $\mathbf{x}_t$ is normally distributed (Theorem \ref{thm:tweedie_iff}). Existing methods use the unconditional score $\nabla \log p_t(\mathbf{x}_t)$ where this assumption does not generally hold. In this work, we instead propose to use the conditional posterior mean $\mathbb{E} \left [ \mathbf{x}_0 | \mathbf{x}_t,\, \mathbf{y} \right ]$ as an estimate for $\mathbf{x}_0$ which requires estimating the forward process score, $\nabla \log p_t (\mathbf{x}_t | \mathbf{x}_0 )$. While the forward process score is intractable during the reverse process (inference), since $\mathbf{x}_0$ is unknown, we can utilize the information contained in $\mathbf{y}$ in a statistically sufficient manner to obtain an estimate (Section \ref{sec:ours}).

\noindent Our \textbf{contributions} are as follows:
\begin{itemize}
    \item We identify a limitation with the use of $\mathbb{E}[\mathbf{x}_0 | \mathbf{x}_t]$ to predict $\mathbf{x}_0$ in inverse problems: the approximation is only exact when $\mathbf{x}_t$ is normally distributed.
    \item We sidestep this by considering the forward process score, $\nabla \log p_t(\mathbf{x}_t | \mathbf{x}_0)$ which we are able to predict using a maximum likelihood estimator. We show theoretically that our estimate is statistically sufficient to the measurement $\mathbf{y}$, and empirically that it is robust to high levels and different types of measurement noise (Figure \ref{fig:noisy_comparison}, Table \ref{table:ffhq_alt_noise}).
    \item We demonstrate how this score can be plugged into any standard sampler (e.g., DDPM), resulting in an algorithm that is simple, noise-robust, neural backpropagation-free, and stable across time steps. Moreover, it achieves improvements in performance on a large selection of inverse problems, datasets and noise levels \footnote{Code for method and experiments can be found \href{https://github.com/diffusion-conditional-sampling/diffusion_conditional_sampling}{here}.}.
\end{itemize}

\section{REVISITING TWEEDIE'S FOR INVERSE PROBLEMS} \label{sec:tweedies_explained}

\begin{figure*}
    \includegraphics[width=\linewidth]{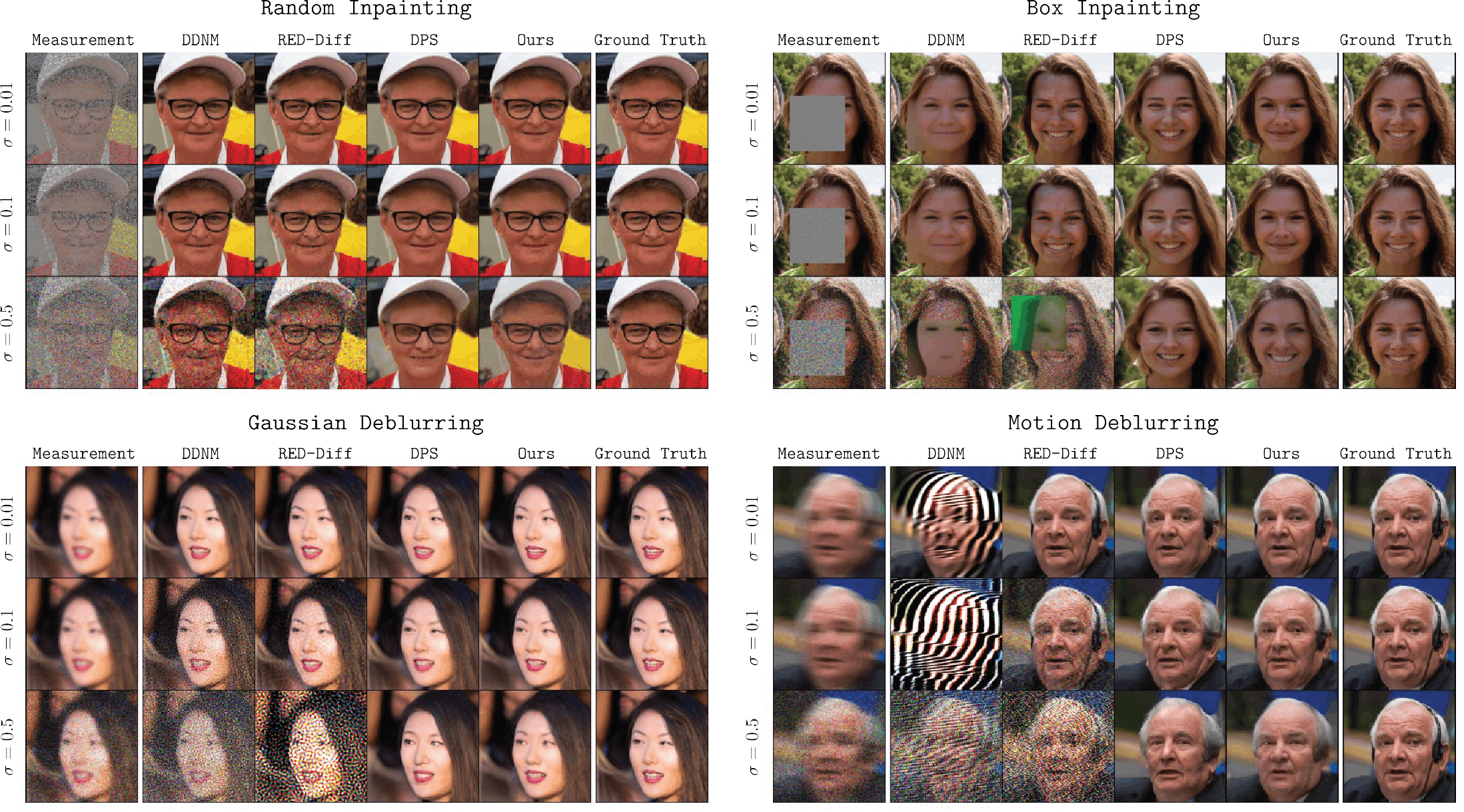}
    \caption{Reconstruction quality at various noise levels $\sigma_\mathbf{y} \in \{0.01, 0.1, 0.5\}$. Our approach strikes a careful balance between quality at each noise level (Table \ref{table:main_quantitative}) and computational cost (Table \ref{table:inverse_solvers_overview}). More examples in Appendix \ref{sec:qualitative}.}
    \label{fig:noisy_comparison}
\end{figure*}

Diffusion models \cite{ho2020denoising} reverse a noise-corrupting forward process, with marginals $\mathbf{x}_t \sim p_t(\mathbf{x}_t | \mathbf{x}_0)$, defined as
\begin{equation}
    \label{eq:xt_marginal}
    \mathbf{x}_t = \sqrt{\alpha_t} \mathbf{x}_0 + \underbrace{\sqrt{1 - \alpha_t}}_{\sigma_t} \mathbf{z}, \hspace{.1in} \mathbf{z} \sim \mathcal{N}(\mathbf{0}, \mathbf{I}),
\end{equation}
and parameterized by a monotonically time-decreasing scalar $\alpha_t$. New samples are generated via the reverse diffusion process which leverages the learned score function $s_\theta = -\sigma_t^{-1} \boldsymbol\epsilon_\theta \approx \nabla\log p_t(\mathbf{x}_t)$ \citep{anderson1982reverse, vincent2011connection, song2020score}. Diffusion-based solvers for inverse problems can be categorized as \textbf{posterior} or \textbf{projection} solvers, and aim to modify the reverse process such that the final variate $\mathbf{x}_0$ coincides with an element of the solution set $\{\mathbf{x}_0: \mathcal{A}(\mathbf{x}_0) \approx \mathbf{y}\}$\footnote{We defer a more extensive discussion on diffusion models diffusion-based inverse problem solvers to Appendix \ref{sec:background}.}. This paradigm is afflicted by a fundamental computability paradox: since the consistency error is only explicitly known at $t=0$ by applying the likelihood function
\begin{equation}
	p(\mathbf{y}|\mathbf{x}_0) 
    \propto \exp\left(-\frac{1}{2\sigma_\mathbf{y}^2} || \mathbf{y} - \mathcal{A}(\mathbf{x}_0) ||_2^2\right),
    \label{eq:likelihood_fn}
\end{equation}
we cannot exactly guide the diffusion process at time $t > 0$ without first solving for $\mathbf{x}_0$. Simultaneously, we cannot generally obtain $\mathbf{x}_0$ without first computing $\mathbf{x}_t$. Accurately estimating $\hat{\mathbf{x}}_0 \approx \mathbf{x}_0$ is a fundamental challenge both types of solvers must contend with to function properly.

In \textbf{posterior} solvers, $\nabla \log p(\mathbf{y} | \mathbf{x}_t)$, is approximated by $\nabla \log p(\mathbf{y} | \hat{\mathbf{x}}_0)$. In \textbf{projection} solvers, this approximation is hidden in the projection step $\mathbf{P}\mathbf{x}_t$, which is driven by a projection on $\hat{\mathbf{x}}_0$, followed by a DDIM step \citep{song2020denoising} that involves $\hat{\mathbf{x}}_0$. In both cases, Tweedie's formula is used to create an estimate for $\mathbf{x}_0$, given the current $\mathbf{x}_t$ by predicting the posterior mean
\begin{equation}
    \mathbb{E}[\mathbf{x}_0 | \mathbf{x}_t] = \int_{\mathbb{R}^d} \mathbf{x}_0 p(\mathbf{x}_0 | \mathbf{x}_t) d\mathbf{x}_0.
    \label{eq:posterior_mean}
\end{equation}
A fundamental limitation of this estimator is that $\mathbb{E}[\mathbf{x}_0 | \mathbf{x}_t]$ only coincides with $\mathbf{x}_0$ when $\mathbf{x}_t$ is marginally normally distributed. We formalize this necessary and sufficient condition below:
\begin{theorem}
    Let $\mathbf{x}_t$ be sampled from a diffusion process (as in Eq. \ref{eq:xt_marginal}). $\mathbb{E}[\mathbf{x}_0 | \mathbf{x}_t] = \mathbf{x}_0$ if and only if $p(\mathbf{x}_t)$ is a simple isotropic Gaussian with mean $\sqrt{\alpha_t} \mathbf{x}_0$ and variance $\sigma_t \mathbf{I}$.
    \label{thm:tweedie_iff}
\end{theorem}
Note that $\mathbf{x}_t$ is almost never Gaussian, since $\mathbf{x}_t$ is distributed as $\phi_\sigma \ast p_\text{data}$. While $\phi_\sigma$ is a simple isotropic Gaussian, $p_\text{data}$ is not --- the data distribution can be arbitrarily non-convex and multimodal.

We visualize this in Figure $\ref{fig:corrected_uncorrected_example}$, where at larger values of $t$, the fidelity of the estimated $\mathbf{x}_0$ is poor, resulting in a low quality prediction that is inconsistent with the measurement $\mathbf{y}$. Ultimately, Eq. \ref{eq:posterior_mean} is a weighted average over \textit{all} data $\mathbf{x} \sim p_\text{data}$, and often cannot properly estimate $\mathbf{x}_0$ without incorporating measurement information from $\mathbf{y}$. Instead, we propose to use the conditional posterior mean
\begin{equation}
    \mathbb{E}[\mathbf{x}_0 | \mathbf{x}_t, \mathbf{y}] = \int_{\mathbb{R}^d} \mathbf{x}_0 p(\mathbf{x}_0 | \mathbf{x}_t, \mathbf{y}) d\mathbf{x}_0,
    \label{eq:cond_posterior_mean}
\end{equation}
which only considers those $\mathbf{x}_0 \sim p_\text{data}$ consistent with $\mathbf{y}$. In the following section, we outline a method for directly incorporating this conditional information into the estimate.

\begin{figure*}[!t]
\begin{minipage}{0.49\textwidth}
\begin{algorithm}[H]
    \centering
    \caption{Diffusion Conditional Sampler (\textbf{DCS})}\label{algorithm}
    \begin{algorithmic}[1]
    \STATE {\bfseries Input:} $\mathbf{y}, \mathcal{A}, \epsilon_\theta$ \hspace{.2in} $|$ {\bfseries Output:} $\mathbf{x}_0$
    \STATE $\mathbf{x}_T \sim \mathcal{N}(\textbf{0}, \mathbf{I})$
    \FOR{$t=T$ {\bfseries to} $1$}    
        \STATE $\epsilon \leftarrow \epsilon_\theta(\mathbf{x}_t)$
        \STATE ${\epsilon_\mathbf{y}} \leftarrow \underset{{\epsilon_\mathbf{y}}}{\arg \texttt{nam}} \; p_t\left(\mathbf{y} | \frac{\mathbf{x}_t - \sigma_t (\epsilon_\theta + \epsilon_\mathbf{y})}{\sqrt{\alpha_t}}\right)$
        \STATE $\mathbf{x}_{t-1} \leftarrow \texttt{ddpm\_step}(\mathbf{x}_t, \epsilon + \epsilon_\mathbf{y})$
    \ENDFOR
    \end{algorithmic}
    \label{alg:ics}
\end{algorithm}
\end{minipage}
\hfill
\begin{minipage}{0.49\textwidth}
\begin{algorithm}[H]
    \centering
    \caption{Noise-aware Maximization (\texttt{nam})}\label{algorithm1}
    \begin{algorithmic}[1]
    \STATE {\bfseries Input:} $\mathbf{y}, \mathcal{A}, \mathbf{x}_t, \boldsymbol{\epsilon}$  \hspace{.15in} $|$ {\bfseries Output:} $\epsilon_y$
    \STATE $\epsilon_y \gets \mathbf{0}$
    \STATE $\hat{\mathbf{x}} \gets \texttt{Tweedie's}(\mathbf{x}_t, \boldsymbol{\epsilon} + \epsilon_y)$
    \WHILE{$2\Phi[-||y - \mathcal{A}[\hat{\mathbf{x}}]||_1/(d\sigma_\mathbf{y})] < \sigma_t$} 
        \vspace{.05in}
        \STATE $\epsilon_y \gets \epsilon_y + \eta \nabla \log p_t\left(\mathbf{y} | \frac{\mathbf{x}_t - \sigma_t (\boldsymbol{\epsilon} + \epsilon_\mathbf{y})}{\sqrt{\alpha_t}}\right)$
        \vspace{.01in}
        \STATE $\hat{\mathbf{x}} \gets \texttt{Tweedie's}(\mathbf{x}_t, \boldsymbol{\epsilon} + \epsilon_y)$
    \ENDWHILE
    \end{algorithmic}
    \label{alg:nam}
\end{algorithm}
\end{minipage}
\end{figure*}

\section{DIFFUSION CONDITIONAL SAMPLING} \label{sec:ours} We propose Diffusion Conditional Sampling (\textbf{DCS}): a novel framework for solving inverse problems with diffusion models via a measurement consistent Tweedie's formula (Eq. \ref{eq:cond_posterior_mean}). At each step, we form a single-parameter measurement model whose maximum likelihood estimator approximates the forward process score $\nabla \log p_t(\mathbf{x}_t | \mathbf{x}_0)$ (Section \ref{sec:ours_likelihood}). This estimator is optimized with a noise-robust, likelihood-based early stopping criterion (Section \ref{sec:ours_nam}). The learned score is then input to a standard DDPM sampler \citep{ho2020denoising}, resulting in Algorithm \ref{alg:ics}. This approach is motivated by both powerful theoretical guarantees (Section \ref{sec:ours_theory}), and yields significant computational advantages (Section \ref{sec:efficiency}).

\subsection{Measurement Likelihood Model}
\label{sec:ours_likelihood}
In our setting, we wish to reverse a diffusion process originating from a single fixed $\mathbf{x}_0$ --- the desired signal $\mathbf{x}$ (Eq. \ref{eq:intro_inverse_problem}), where
\begin{equation}
	p_*(\mathbf{x}_t) = \mathcal{N}(\mathbf{x}_t; \sqrt{\alpha_t} \mathbf{x}_0, (1 - \alpha_t) \mathbf{I}).
    \label{eq:single}
\end{equation}
In standard unconditional sampling, this is also known as the forward process $p_t(\mathbf{x}_t | \mathbf{x}_0)$ \citep{ho2020denoising}. However, when running the reverse process (unconditional sampling) this is a non-trivial quantity to evaluate, as $\mathbf{x}_0$ is unknown. We apply Tweedie's to Eq. \ref{eq:single} and solve for the forward process score via a closed form expression for the likelihood (Eq. \ref{eq:likelihood_fn}):
\begin{align} 
    \log\  &p(\mathbf{y}|\mathbf{x}_0({\epsilon_\mathbf{y}}, \mathbf{x}_t)) \propto \nonumber\\
    &-\frac{1}{2\sigma_\mathbf{y}^2} \left|\left| \mathbf{y} - \mathcal{A}\left(\frac{1}{\sqrt{\alpha_t}}[\mathbf{x}_t + \sigma_t^2 \underbrace{\nabla_{\mathbf{x}_t} \log p_*(\mathbf{x}_t)}_{\nabla \log p_t(\mathbf{x}_t | \mathbf{x}_0)}]\right) \right|\right|_2^2
    \label{eq:our_likelihood}
\end{align}

Since $p_t(\mathbf{x}_t | \mathbf{x}_0)$ is distributed as an isotropic Gaussian by construction, knowledge of the true conditional posterior mean would recover $\mathbf{x}_0$ exactly (Theorem \ref{thm:tweedie_iff}). Introducing the parameterization,
\begin{equation}
    \nabla \log p_{t}(\mathbf{x}_t | \mathbf{x}_0) = -\sigma_t^{-1}[\epsilon_\theta(\mathbf{x}_t, t) + \epsilon_\mathbf{y}],
\end{equation}
we can solve for our single parameter $\epsilon_y$ by maximizing the joint likelihood between the measurement $\mathbf{y}$ and our parameter ${\epsilon_\mathbf{y}}$. This forms our estimator for $\nabla \log p_t(\mathbf{x}_t | \mathbf{x}_0)$:
\begin{align}
\label{eq:data_conditional_correction}
    s_\text{corrected} = -\sigma_t^{-1} & \left [ \epsilon_\theta(\mathbf{x}_t, t) + \epsilon_\mathbf{y}^* \right ] ,
\end{align}
where
\begin{align}
    \label{eq:dcs_mle}
    \epsilon_\mathbf{y}^* &=\\
    =&\underset{\epsilon_\mathbf{y}}{\arg\max}\; \frac{-1}{2\sigma_\mathbf{y}^2} \left|\left| \mathbf{y} - \mathcal{A}\left(\frac{\mathbf{x}_t - \sigma_t [\epsilon_\theta(\mathbf{x}_t, t) + \epsilon_\mathbf{y}]}{\sqrt{\alpha_t}}\right) \right|\right|_2^2.
\end{align}
Finally, $s_\text{corrected}$ can then be input to any standard diffusion model sampler. 

Now we turn to estimating $\epsilon_\mathbf{y}^*$. Given the often noisy and ill-posed nature of Eq. \ref{eq:dcs_mle} (Appendix \ref{sec:noninvertible_a}), we seek to select from the solution set $\{\epsilon_\mathbf{y} : \mathcal{A}[\hat{\mathbf{x}}_0] \approx \mathbf{y} \}$ through a noise-aware maximization algorithm, which we outline below.

\begin{figure*}[h!]
\centering
\begin{minipage}{0.48\textwidth}
\includegraphics[width=\linewidth]{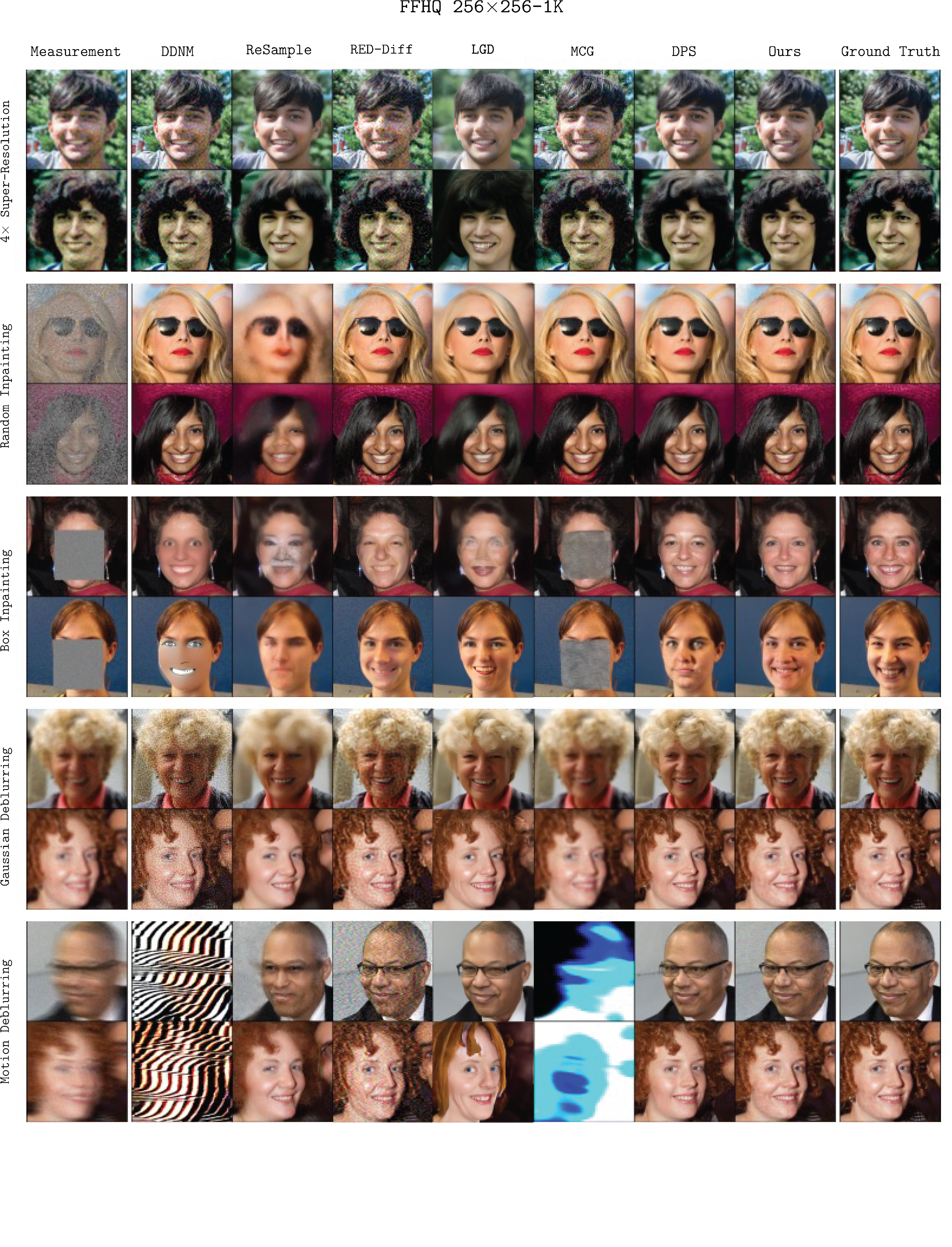}
\end{minipage}
\hspace{.1in}
\begin{minipage}{0.48\textwidth}
\includegraphics[width=\linewidth]{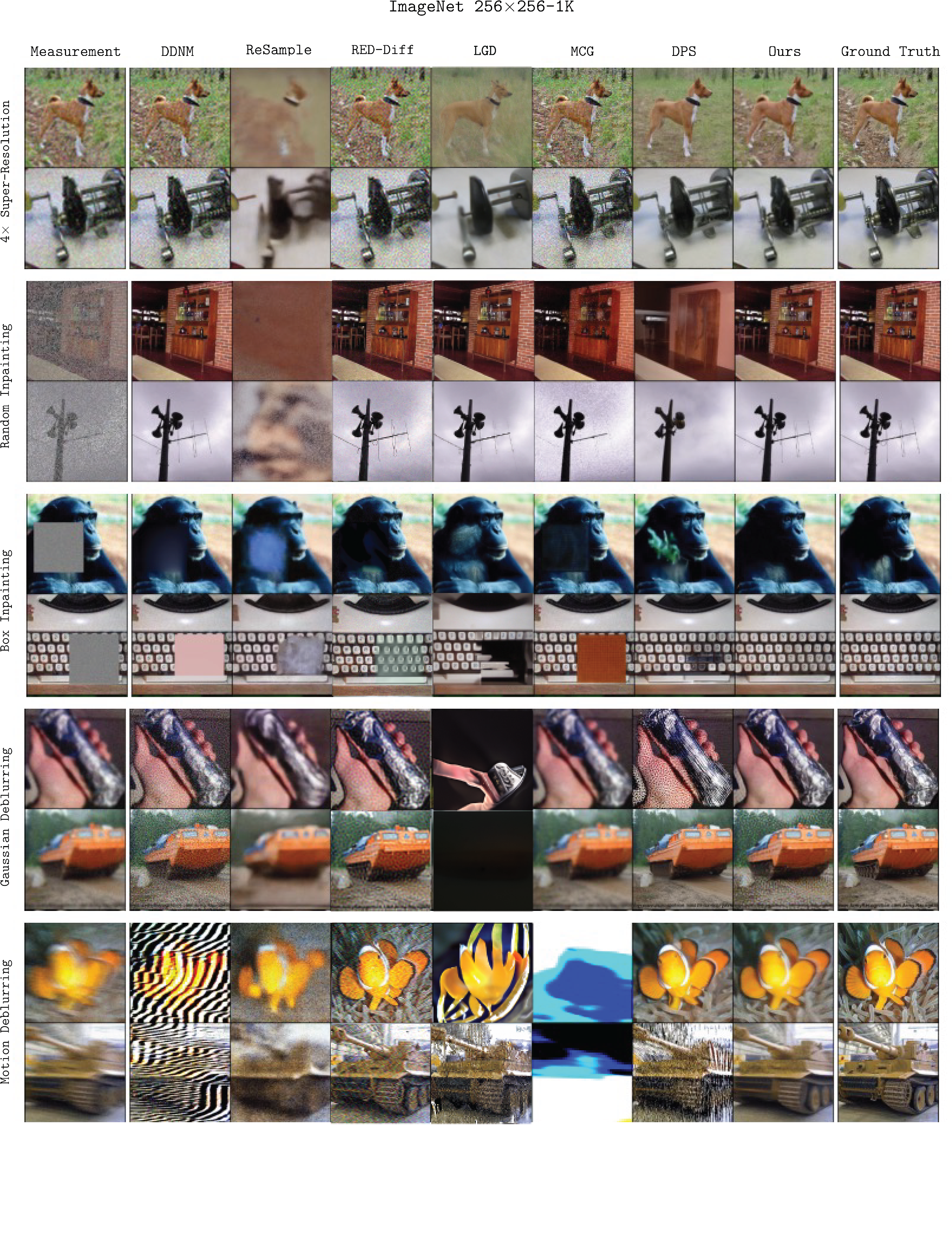}
\end{minipage}
\caption{Qualitative comparison of our proposed method against related work on FFHQ 256$\times$256-1K (left) and ImageNet 256$\times$256-1K (right). Further comparisons can be found in Appendix \ref{sec:qualitative}.}
\label{fig:qualitative}
\end{figure*}

\subsection{Noise-Aware Maximization}
\label{sec:ours_nam}
We propose a \textbf{n}oise-\textbf{a}ware \textbf{m}aximization scheme (\texttt{nam}) to improve stability across noise levels. 
Given a single noisy measurement $\mathbf{y} = \mathcal{A}[\mathbf{x}] + \boldsymbol{\eta}$, there is a high risk of overfitting to noise $\boldsymbol{\eta}$ (Figures \ref{fig:overfit_underfit_example} and \ref{fig:noisy_comparison}). To mitigate this problem, we propose a maximization scheme with a specialized early stopping criterion based on the measurement likelihood. We leverage the intuition that the corrected forward process score should yield a prediction where a vector of residuals,
\begin{equation}
    \texttt{res} = \mathbf{y} - \mathcal{A}[\hat{\mathbf{x}}_0]
\end{equation}
should be i.i.d. normally distributed with variance $\sigma_\mathbf{y}^2$. In other words, each index of \texttt{res} should come from the same distribution as an index of $\boldsymbol{\eta}$. Let this be the null hypothesis $\mathbb{H}_0$ --- we seek to end the likelihood maximization process as soon as $\mathbb{H}_0$ holds.

Formally, we optimize Eq. \ref{eq:our_likelihood} until the likelihood of the alternate hypothesis, $\mathbb{H}_1$ (that \texttt{res} is \textit{not} distributed as $\eta$),  is below a desired threshold $p_\text{critical}$. Since overfitting is more problematic at the end of sampling ($t \approx 0$) than the beginning of sampling ($t \approx T$), we set $p_\text{critical}$ dynamically as a function of $t$, namely $p_\text{critical}(t) = \sigma_t$. This scheme is heavily inspired by the classical two-sided $\mathbf{z}$-test \citep{hogg2013introduction} with $d$ samples, where $d$ is the dimensionality of the image. We use the early-stopping criterion at each time $t$
\begin{equation} \label{eq:nam_test}
    \mathbb{P}(|\xi| > |\texttt{res}| \big| \mathbb{H}_0) = 2\Phi(-|\texttt{res}| / \sigma_\mathbf{y}) < \sigma_t,
\end{equation}
where $\xi_i \overset{\text{iid}}{\sim} \mathcal{N}(\mathbf{0}, \sigma_\mathbf{y}^2)$ and $\Phi$ is the CDF of a standard normal distribution. This differs from a classical $\mathbf{z}$-test since we are not seeking to reject the null hypothesis, but optimizing until the null hypothesis can no longer be rejected with sufficiently high probability. The full noise-aware maximization algorithm can be summarized by Algorithm \ref{alg:nam}. Since our loss function (Eq. \ref{eq:our_likelihood}) is quadratic, our proposed \texttt{nam} has worst-case linear convergence guarantees due to classical results with gradient descent \citep{boyd2004convex,ryu2016primer} with linear $\mathcal{A}$.

\subsection{Theory}
\label{sec:ours_theory}
We highlight two key theoretical properties of our sampler, the correctness of the Tweedie's approximation and the sufficiency of the resulting score with respect to $\mathbf{y}$.

\paragraph{Correctness} Our algorithm makes use of the conditional variant of Tweedie's formula, which obeys a very similar set of rules as Tweedie's (Theorem \ref{thm:tweedie_iff}).
\begin{theorem}
    Let $\mathbf{x}_t$ be sampled from a conditional diffusion process given $\mathbf{y}$ (as in Eq. \ref{eq:xt_marginal}, with $\mathbf{x}_0 \sim p(\mathbf{x}_0 | \mathbf{y})$). $\mathbb{E}[\mathbf{x}_0 | \mathbf{x}_t, \mathbf{y}] = \mathbf{x}_0$ if and only if $p(\mathbf{x}_t | \mathbf{x}_0)$ is a simple isotropic Gaussian with mean $\sqrt{\alpha_t} \mathbf{x}_0$ and variance $\sigma_t \mathbf{I}$.
    \label{thm:tweedie_cond_iff}
\end{theorem}
This differs from existing approaches in two key ways. First, we incorporate conditional information from $\mathbf{y}$ directly into the denoising process. Second, rather than $p_t(\mathbf{x}_t)$, we consider $p_t(\mathbf{x}_t | \mathbf{x}_0)$, which is an isotropic Gaussian distribution by construction. Therefore, given $\nabla \log p_t(\mathbf{x}_t | \mathbf{x}_0)$, Theorem \ref{thm:tweedie_cond_iff} tells us that the application of Tweedie's formula in Eq. \ref{eq:our_likelihood} will exactly recover $\mathbf{x}_0$. 

\paragraph{Sufficiency} Even with the optimization framework in Eq. \ref{eq:data_conditional_correction}, $\nabla \log p_t(\mathbf{x}_t | \mathbf{x}_0)$ can only be computed up to the information present in $\mathbf{y}$. However, we show that this is provably the best estimate in the following sense: $\epsilon_\textbf{y}^*$ (and therefore $s_{\text{corrected}}$) is a sufficient statistic for the ground truth $\mathbf{x}_0$ given measurement $\mathbf{y}$ under regularity conditions on $\mathcal{A}$ and $\boldsymbol\eta$:
\begin{theorem}[${\epsilon_\mathbf{y}}_*$ is a sufficient statistic]
\label{thm:sufficiency}
Let $\mathbf{y} = \mathcal{A}(\mathbf{x}_0) + \boldsymbol{\eta}$ be an observation from the forward measurement model, such that $\boldsymbol{\eta} = 0$, or $\mathcal{A}$ is linear. If ${\epsilon_\mathbf{y}}_*$ is defined as in Eq. \ref{eq:dcs_mle}, then $p(\mathbf{y} |  {\epsilon_\mathbf{y}}_*, \mathbf{x}_0 ) = p(\mathbf{y} | {\epsilon_\mathbf{y}}_*)$.
\end{theorem}
We extend this result to a broader class of operators in Theorem \ref{thm:general_sufficiency}. In this sense, $\textbf{DCS}$ effectively closes an information "leak" by ensuring that the only information about $\mathbf{x}_0$ lost in the sampling process is solely that which is irrevocably destroyed by the operator $\mathcal{A}$.

\subsection{Computational Efficiency}
\label{sec:efficiency}
Empirically, \textbf{DCS} enjoys two main computational advantages.
First, it does not need to compute expensive gradients of the score function. Second, it boasts stable performance across choices of $T$ due to the linearity of the forward process diffusion process.

\begin{figure}
    \centering
    \includegraphics[width=.9\linewidth]{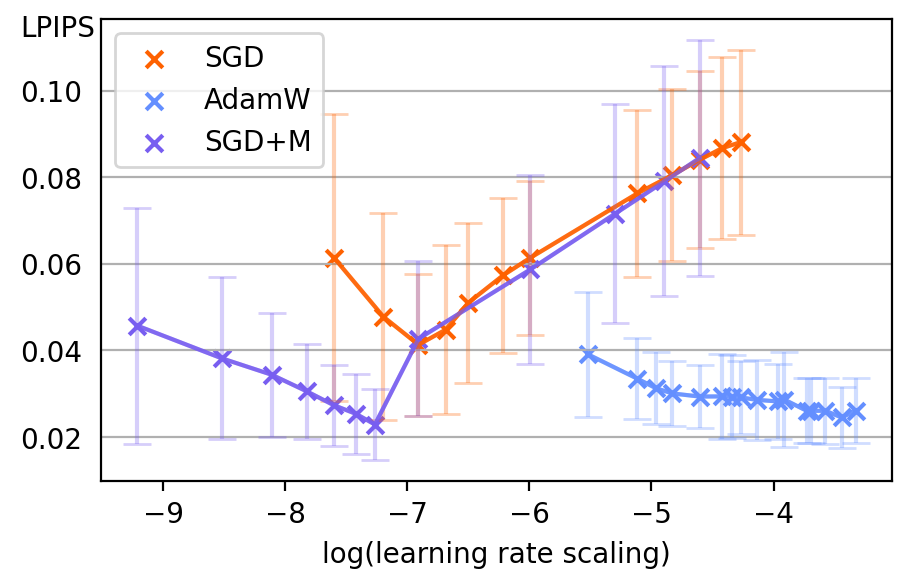}
    \caption{\textbf{DCS} performance across several choices of optimizers. LPIPS score of the predicted $\mathbf{x}_0$ images is plotted against the natural log of learning rate scaling factor for each optimizer.}
    \label{fig:optimizer_ablation}
\end{figure}

\paragraph{No Expensive $\nabla s_\theta(\mathbf{x}_t, t)$ Evaluations} A drawback of many existing algorithms is the need to compute gradients of the score network during sampling (Table \ref{table:inverse_solvers_overview}). This is the most expensive computation in the diffusion step, increasing the runtime of the algorithm by $\mathbf{2}$-$\mathbf{3}\times$. However, this is unavoidable in these solvers without sacrificing performance (Section \ref{sec:jf_main}).

\paragraph{A Near-Linear Reverse Process} As \textbf{DCS} models $\mathbb{E}(\mathbf{x}_0 | \mathbf{x}_t, \mathbf{y})$, it is able to sample approximately from the reverse diffusion process conditioned on the measurement, which reverses the forward process defined in Eq. \ref{eq:xt_marginal}. If the forward process score was able to be exactly estimated, Tweedie's recovers $\mathbf{x}_0$, and the diffusion process can be solved in a single step. In reality, our approximation of this process is correct up to the information about $\mathbf{x}_0$ present in $\mathbf{y}$ (Theorem \ref{thm:general_sufficiency}), under the assumptions detailed in the previous section.

In Figure \ref{fig:robustness}, we validate the robustness of our algorithm to the total diffusion steps ($T$) with the super-resolution task on a subset of the FFHQ $256\times 256$ dataset. We compare against DPS \cite{chung2023diffusion}, DPS-JF (a neural backpropagation-free variant of DPS), and DDNM \cite{wang2022zero} with additive Gaussian noise at standard deviation $\sigma_{\mathbf{y}}=0.05$.

\begin{figure}
    \centering
    \includegraphics[width=0.9\linewidth]{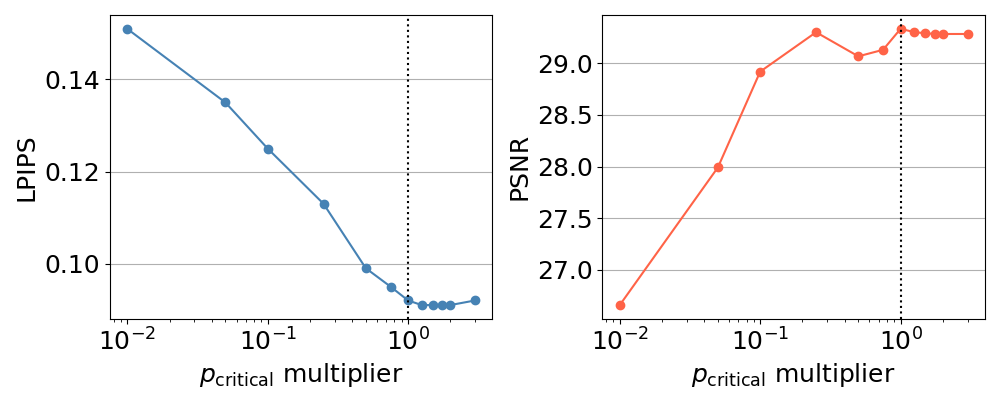}
    \caption{Testing the sensitivity of the $p_\text{critical}$ parameter. \textbf{DCS} is evaluated on the FFHQ dataset for the SRx4 task, varying a multiplier displayed along the x-axis. Our choice of $p_\text{critical}$, here represented by a multiplier of $1$, is signified by a dotted line.}
    \label{fig:pcrit_ablation}
\end{figure}

\begin{figure}
    \centering
    \includegraphics[width=.9\linewidth]{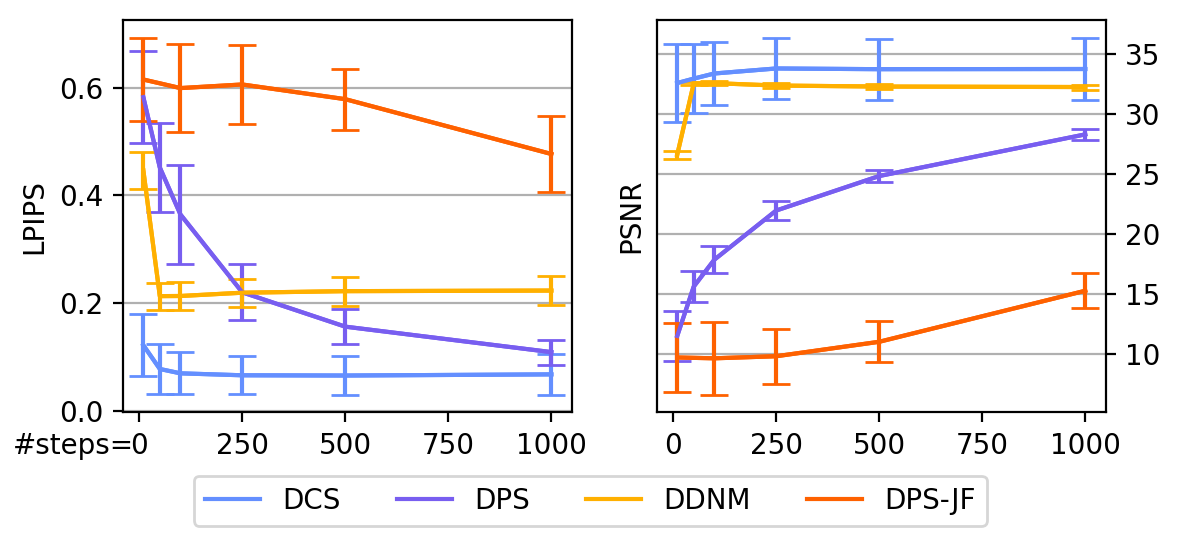}
    \caption{A study on the effect of $T$ on solver performance. \textbf{DCS} method remains nearly invariant to $T$ due to the near-linearity of the forward process diffusion process.}
    \label{fig:robustness}
\end{figure}

\begin{table*}[t!]
    \setlength{\tabcolsep}{1mm}
    \fontsize{8pt}{6pt}\selectfont
    \centering
    \begin{adjustbox}{max width=\linewidth}
        \begin{tabular}{@{}lccccccccccccccc@{}}
            \toprule
            \textbf{FFHQ} & \multicolumn{3}{c}{SR $\times 4$} & \multicolumn{3}{c}{Random Inpainting} & \multicolumn{3}{c}{Box Inpainting} & \multicolumn{3}{c}{Gaussian Deblurring} & \multicolumn{3}{c}{Motion Deblurring}\\
            \cmidrule(r){2-4}\cmidrule(lr){5-7}\cmidrule(l){8-10}\cmidrule(l){11-13}\cmidrule(l){14-16}
            $\sigma_\mathbf{y}=0.01$ & LPIPS $\downarrow$ & PSNR $\uparrow$ & FID $\downarrow$ & LPIPS $\downarrow$ & PSNR $\uparrow$ & FID $\downarrow$ & LPIPS $\downarrow$ & PSNR $\uparrow$ & FID $\downarrow$ & LPIPS $\downarrow$ & PSNR $\uparrow$ & FID $\downarrow$ & LPIPS $\downarrow$ & PSNR $\uparrow$ & FID $\downarrow$ \\
            \midrule
            \textbf{DCS} (Ours) & \textbf{0.137} & \textbf{30.14} & \textbf{19.45} & \textbf{0.024} & \textbf{34.84} & \textbf{21.19} & \textbf{0.088} & \underline{25.11} & \textbf{19.25} & \textbf{0.103} & \underline{28.69} & \textbf{22.62} & \textbf{0.087} & \textbf{29.48} &  \textbf{26.67} \\
            \midrule
            MCG         & \underline{0.144} & 24.84 & \underline{31.47} & 0.073 & 30.59 & \underline{22.22} & 0.453 & 15.44 & 185.54 & 0.209 & 23.51 & 67.88 & 0.217 & 22.93 & 292.1 \\
            DDNM        & 0.208             & \underline{26.28} & 51.33 & \underline{0.040} & \underline{33.08} & 23.35  &     0.209 &     18.12 & 88.32 & 0.235 & 26.09 & 71.47 & 0.424 & 14.22 & 250.9 \\
            LGD-MC & 0.238 &     23.45 &     39.55 &      0.272 &     23.46 &     57.70 &      0.372 &     15.45 &     86.00 &      0.405 &     18.78 &     64.31 &      0.520 &     13.90 &    106.0\\
            DPS         & 0.163             & 25.91 & 33.21 & 0.105 & 29.54 & 29.72 & 0.113 & \underline{23.52} & 24.41 & \underline{0.129} & 26.48 & \underline{26.85} & \underline{0.159} & 24.41 & \underline{29.84} \\
            RED-Diff    &  0.178            &     29.81 &     45.68 &      0.035 &     33.72 &     25.03 &      \underline{0.090} &     \textbf{25.20} &      \underline{19.98} &      0.234 &     \textbf{29.72} &     52.09 &      0.191 &     \underline{29.14} &    116.9 \\
            \toprule
            \textbf{FFHQ} & \multicolumn{3}{c}{SR $\times 4$} & \multicolumn{3}{c}{Random Inpainting} & \multicolumn{3}{c}{Box Inpainting} & \multicolumn{3}{c}{Gaussian Deblurring} & \multicolumn{3}{c}{Motion Deblurring}\\
            \cmidrule(r){2-4}\cmidrule(lr){5-7}\cmidrule(l){8-10}\cmidrule(l){11-13}\cmidrule(l){14-16}
            $\sigma_\mathbf{y}=0.1$ & LPIPS $\downarrow$ & PSNR $\uparrow$ & FID $\downarrow$ & LPIPS $\downarrow$ & PSNR $\uparrow$ & FID $\downarrow$ & LPIPS $\downarrow$ & PSNR $\uparrow$ & FID $\downarrow$ & LPIPS $\downarrow$ & PSNR $\uparrow$ & FID $\downarrow$ & LPIPS $\downarrow$ & PSNR $\uparrow$ & FID $\downarrow$ \\
            \midrule
            \textbf{DCS} (Ours)                 & \textbf{0.175} & \textbf{24.88} & \textbf{30.11} & \textbf{0.149} & \underline{27.54} & \textbf{32.80} & \underline{0.163} & \textbf{23.22} & \textbf{26.44} & \textbf{0.176} & \underline{25.96} & \textbf{26.08} & \underline{0.224} & \textbf{24.61} & \textbf{31.40} \\
            \midrule
            MCG                  & 0.546 & 20.44 & 102.6 & 0.227 & 26.00 & 50.40 & 0.579 & 15.30 & 207.2 & 0.429 & 25.80 & 69.29 & 0.973 & -7.104 & 295.3 \\
            DDNM                 &      0.623 &     21.49 &    145.9 &      0.179 &     24.96 &     39.18 &      0.334 &     19.20 &     72.11 &      1.220 &     10.73 &    176.8 &      0.739 &      5.099 &    524.0 \\
            LGD-MC & 0.256 &     22.31 &     39.58 &      0.288 &     22.22 &     56.05 &      0.384 &     15.38 &     87.72 &      0.415 &     18.30 &     66.04 &      0.524 &     13.65 &    105.4\\
            DPS                  &      \underline{0.185} & \underline{24.79} & \underline{35.46} & 0.157 & \underline{26.72} & 35.24 & \textbf{0.158} & 22.58 & \underline{32.47} & \underline{0.180} & 24.720 & 33.53 & \textbf{0.212} & \underline{22.41} & \underline{35.09} \\
            RED-Diff     & 0.665 &     22.10 &    155.1 &      \underline{0.155} &     \textbf{28.62} &     \underline{34.78} &      0.298 &     \underline{22.96} &     61.14 &      0.447 &     \textbf{26.93} &    106.3 &      0.423 &     24.16 &    120.1\\
            \bottomrule
            \toprule
            \textbf{ImageNet}& \multicolumn{3}{c}{SR $\times 4$} & \multicolumn{3}{c}{Random Inpainting} & \multicolumn{3}{c}{Box Inpainting} & \multicolumn{3}{c}{Gaussian Deblurring} & \multicolumn{3}{c}{Motion Deblurring}\\
            \cmidrule(r){2-4}\cmidrule(lr){5-7}\cmidrule(l){8-10}\cmidrule(l){11-13}\cmidrule(l){14-16}
            $\sigma_\mathbf{y}=0.01$ & LPIPS $\downarrow$ & PSNR $\uparrow$ & FID $\downarrow$ & LPIPS $\downarrow$ & PSNR $\uparrow$ & FID $\downarrow$ & LPIPS $\downarrow$ & PSNR $\uparrow$ & FID $\downarrow$ & LPIPS $\downarrow$ & PSNR $\uparrow$ & FID $\downarrow$ & LPIPS $\downarrow$ & PSNR $\uparrow$ & FID $\downarrow$ \\
            \midrule
            \textbf{DCS} (Ours) & \textbf{0.238} & 23.45 & \textbf{39.41} & 0.142 & 26.06 & 34.46 & \textbf{0.230} & \textbf{20.63} & \underline{37.11} & \textbf{0.253} & 24.22 & \textbf{38.96} & \textbf{0.203} & \textbf{24.62} & \textbf{38.63}\\
            \midrule
            MCG & 0.638 & 15.62 & 89.39 & 0.198 & 24.34 & 35.19 & 0.273 & 16.68 & 80.35 & 0.645 & 21.18 & 124.6 &  0.980 & -5.726 &  231.1 \\
            DDNM & 0.333 & \textbf{25.16} & 51.33 & \textbf{0.084} & \textbf{28.35} & \underline{20.27} & 0.258 & 17.42 & 85.41 & 0.456 & \underline{24.35} & 67.98 & 0.694 & 5.721 &  304.2\\
            LGD-MC    & 0.662 &     14.460 &    113.5 &      0.650 &     14.57 &    129.8 &      0.696 &     11.63 &    133.9 &      0.796 &     10.46 &    165.7 &      0.807 &      9.609 &    184.7\\
            DPS & \underline{0.309} & 23.99 & \underline{49.81} & 0.266 & 25.05 & 38.87 & 0.301 & 18.76 & \textbf{34.85} & 0.493 & 19.14 & \underline{61.59} &  0.460 & 18.65 & \underline{53.21}\\
            RED-Diff & 0.386 &     \underline{25.07} &     57.06 &      \underline{0.090} &     \underline{28.17} &     \textbf{16.71} &      \underline{0.239} &     \underline{19.99} &     54.38 &      0.459 &     \textbf{24.70} &     68.71 &      0.376 &     \underline{23.66} &     55.77 \\
            \bottomrule
            \toprule
            \textbf{ImageNet} & \multicolumn{3}{c}{SR $\times 4$} & \multicolumn{3}{c}{Random Inpainting} & \multicolumn{3}{c}{Box Inpainting} & \multicolumn{3}{c}{Gaussian Deblurring} & \multicolumn{3}{c}{Motion Deblurring}\\
            \cmidrule(r){2-4}\cmidrule(lr){5-7}\cmidrule(l){8-10}\cmidrule(l){11-13}\cmidrule(l){14-16}
            $\sigma_\mathbf{y}=0.1$ & LPIPS $\downarrow$ & PSNR $\uparrow$ & FID $\downarrow$ & LPIPS $\downarrow$ & PSNR $\uparrow$ & FID $\downarrow$ & LPIPS $\downarrow$ & PSNR $\uparrow$ & FID $\downarrow$ & LPIPS $\downarrow$ & PSNR $\uparrow$ & FID $\downarrow$ & LPIPS $\downarrow$ & PSNR $\uparrow$ & FID $\downarrow$ \\
            \midrule
            \textbf{DCS} (Ours) & \textbf{0.402} & \textbf{22.99} & \textbf{48.21} & \textbf{0.166} & \textbf{26.04} & \underline{34.47} & \textbf{0.243} & \textbf{19.70} & \textbf{46.03} & \textbf{0.407} & \underline{22.28} & \textbf{51.13} & \textbf{0.435} & \underline{20.43} & \textbf{61.48} \\
            \midrule
            MCG                  & 0.886 & 14.01 & 145.1 & 0.459 & 19.92 & 78.86 & 0.433 & 15.63 & 124.0 & 0.650 & \underline{22.00} & 117.4 & 0.984 & -6.868 & 231.3 \\
            DDNM                 & 0.751 & \underline{20.98} & 133.3 & 0.1693 & \underline{25.63} & 35.72 & 0.400 & 18.06 & 110.8 & 1.221 & 9.602 & 202.7 & 0.783 & 5.009 & 350.1 \\
            LGD-MC &   0.671 &     14.01 &    116.51 &      0.661 &     14.19 &    131.2 &      0.701 &     11.53 &    134.58 &      0.804 &     10.44 &    167.3 &      0.806 &      9.587 &    185.4 \\
            DPS                  & \underline{0.540} & 18.63 & \underline{85.06} & 0.506 & 20.10 & 82.74 & 0.479 & 18.03 & \underline{83.06} & \underline{0.412} & 20.57 & \underline{65.07} & \underline{0.450} & 18.91 & \underline{75.65} \\ 
            RED-Diff             & 0.747 &     20.66 &    136.35 &     \underline{0.167} &     25.38 &     \textbf{32.99} &      \underline{0.374} &     \underline{19.68} &     88.20 &      0.660 &     \textbf{23.19} &    110.9 &      0.591 &     \textbf{21.27} &    138.8 \\
            \bottomrule
        \end{tabular}
    \end{adjustbox}
    \caption{Quantitative comparison on FFHQ 256x256-1K and ImageNet-1K datasets across various inverse problem tasks and noise levels ($\sigma_{\mathbf{y}} \in \{ 0.01, 0.1 \}$).}
    \label{table:main_quantitative}
\end{table*}

\section{EXPERIMENTS}
\label{sec:experiments}
We examine the empirical performance of \textbf{DCS} across a variety of natural image based inverse problems, against a range of state-of-the-art methods. Quantitatively, we leverage three key metrics to evaluate the quality of signal recovery: Learned Perceptual Image Patch Similarity (LPIPS), peak signal-to-noise ratio (PSNR), and Frechet Inception Distance (FID).

We run \textbf{DCS} and the other methods listed in Table \ref{table:inverse_solvers_overview} on the FFHQ-256 \citep{karras2019style}, \citep{kazemi2014one} and ImageNet \citep{deng2009imagenet} datasets. For the prior network $\epsilon_\theta$, we use the corresponding pretrained model weights from \citep{chung2022diffusion}.

We examine five operator inversion tasks: Super-Resolution, Gaussian Deblurring, Motion Deblurring, Random Inpainting, and Box Inpainting. We first run experiments with additive Gaussian noise of standard deviation $\sigma_{\mathbf{y}} \in \{ 0.01, 0.1 \}$. We also present quantitative results on FFHQ and ImageNet in Table \ref{table:main_quantitative}, and a qualitative comparison in Figures \ref{fig:qualitative}, \ref{fig:noisy_comparison}. We delegate further experiments, such as evaluations on subsets of FFHQ used in other works, additional qualitative comparisons, and details of the implementation to Appendix \ref{sec:additional_experiments}, \ref{sec:implementation_details} and \ref{sec:qualitative}.

We find that \textbf{DCS} either outperforms, or is comparable to all existing methods. While some methods excel at certain metrics in certain tasks and fail to recover the signal at other times, \textbf{DCS} consistently performs well across experiments. For example, \textbf{DCS} is one of few methods that has reasonable results on Motion Deblurring at high noise levels. DDNM and RED-Diff, on the other hand, are powerful across inpainting tasks in general, but fail to perform Motion Deblurring and show underwhelming qualitative performance on many tasks, especially in the high noise regime (Figures \ref{fig:noisy_comparison}, \ref{fig:qualitative}). We find that projection methods in general overfit in the high noise regime. Other methods such as LGD-MC degrade less severely in the presence of noise, however do not perform as well in lower noise levels indicating underfitting. It is apparent from both quantitative and qualitative comparisons that \textbf{DCS} and DPS strike a clearer balance between over- and underfitting.
We also notice that \textbf{DCS} provides a significant speedup and reduction in memory footprint compared to most methods, as noted in Table \ref{table:inverse_solvers_overview}. We achieve this by not requiring backpropagation of the score network, as well as limiting the required number of neural function evaluations by using the measurement-consistent version of Tweedie's formula.


\subsection{Ablation on Measurement Noise Type}

We consider \textbf{DCS}'s robustness to non-Gaussian noise types. To solve inverse problems with different noise types, we modify both the measurement model (Eq. \ref{eq:likelihood_fn}), and NAM threshold.

We first extend \textbf{DCS} (and NAM) to operate under a Poisson noise model. Following \citep{chung2023diffusion}, we change the Gaussian forward measurement model to a Gaussian approximation of the Poisson likelihood (see \citep{chung2023diffusion}, Eqs. (17)-(19)), and use the Poisson CDF in the NAM threshold instead of the Gaussian CDF Eq. \ref{eq:nam_test}:
$$F(\mathbf{y}_j | \hat{\mathbf{x}}_0) = \sum_{k=0}^{\lfloor \mathbf{y}_j \rfloor }  \frac{[\mathcal{A}( \hat{\mathbf{x}}_0)]_j^k \exp \left( -[\mathcal{A}( \hat{\mathbf{x}}_0)]_j \right)}{k!}.$$

We next consider a Laplace noise model. We make analogous modifications, switching both the forward measurement model, and the CDF in NAM to utilize the Laplace CDF:
$$F(\mathbf{y}_j | \hat{\mathbf{x}}_0) = \frac{1}{2} + \frac{1}{2} \text{sgn}(y_j - \mathcal{A}(\hat{\mathbf{x}}_0))  e^{-|y_j - \mathcal{A}(\hat{\mathbf{x}}_0)|/b}.$$

We evaluate our method using the FFHQ 256x256 dataset, with Poisson rate parameter $\lambda=1$, and Laplace parameter $b = 0.05$ on all five tasks considered in this work, and display results in Table \ref{table:ffhq_alt_noise}. Under Poisson noise, \textbf{DCS} consistently outperforms or achieves parity with DPS across metrics. For Laplace noise, \textbf{DCS} achieves parity with DPS across metrics. We also observe that \textbf{DCS} retains the computation advantage noted in Table \ref{table:inverse_solvers_overview} when running with different types of measurement noise.

\begin{table*}[t!]
    \setlength{\tabcolsep}{1mm}
    \fontsize{8pt}{6pt}\selectfont
    \centering
    \begin{adjustbox}{max width=\linewidth}
        \begin{tabular}{@{}llcccccccccc@{}}
            \toprule
            \textbf{FFHQ} & & \multicolumn{2}{c}{SR $\times 4$} & \multicolumn{2}{c}{Random Inpainting} & \multicolumn{2}{c}{Box Inpainting} & \multicolumn{2}{c}{Gaussian Deblurring} & \multicolumn{2}{c}{Motion Deblurring}\\
            \cmidrule(r){3-4}\cmidrule(lr){5-6}\cmidrule(l){7-8}\cmidrule(l){9-10}\cmidrule(l){11-12}
             Noise Type & Algorithm & LPIPS $\downarrow$ & PSNR $\uparrow$ & LPIPS $\downarrow$ & PSNR $\uparrow$ & LPIPS $\downarrow$ & PSNR $\uparrow$ & LPIPS $\downarrow$ & PSNR $\uparrow$ & LPIPS $\downarrow$ & PSNR $\uparrow$ \\
            \midrule
            \textit{Poisson} & \textbf{DCS} (Ours) & \textbf{0.120} & \textbf{28.301} & \textbf{0.089} & \textbf{29.344} & \textbf{0.131} & 23.072 & \textbf{0.110} & \textbf{28.028} & \textbf{0.130} & \textbf{26.437}\\
            & DPS & 0.166 & 24.143 & 0.141 & 26.062 & 0.133 & \textbf{24.480} & 0.126 & 26.600 & 0.155 & 24.034\\
            \midrule
            \textit{Laplace} & \textbf{DCS} (Ours) & 0.099 & 28.348 & 0.126 & 26.129 & \textbf{0.123} & 24.395 & \textbf{0.105} & 28.990 & \textbf{0.114} & \textbf{28.026}\\
            & DPS & \textbf{0.089} & \textbf{29.608} & \textbf{0.120} & \textbf{28.362} & 0.125 & \textbf{24.576} & 0.112 & \textbf{29.469} & 0.116 & 27.682\\
        \end{tabular}
    \end{adjustbox}
    \caption{Quantitative comparison on FFHQ 256x256-1K across various inverse problem tasks for Poisson noise with $\lambda = 1$ and Laplace noise with $b=0.05$.}
    \label{table:ffhq_alt_noise}
\end{table*}

\subsection{Ablation on the Noise-aware Maximization Optimizer}

We investigate how the choice of optimizer and parameters affects the noise-aware maximization algorithm in \textbf{DCS}. We note that the flexibility of using an optimizer enables us to make use of the stopping criterion as detailed in Section \ref{sec:ours}. In Figure \ref{fig:optimizer_ablation} we run \textbf{DCS} with AdamW \cite{loshchilov2017fixing}, SGD with momentum, and vanilla SGD to solve the SRx4 task on a subset of FFHQ. Runs of each optimizer at learning rate scaling factors are displayed to show the best performance, ensuring a fair comparison. It is clear in Figure \ref{fig:optimizer_ablation} that the addition of a momentum term to the optimization process (both present in AdamW and SGD with momentum) can attain a higher level of image fidelity and solver stability than vanilla SGD. This provides empirical evidence for optimizer bias having an effect on solver performance in \textbf{DCS}. We see from this experiment that AdamW produces the most consistent results across learning rates, which motivates its use in our implementation.

\subsection{Sensitivity of NAM stopping threshold}

The choice of $p_{\text{critical}}$ is statistically inspired, so we confirm it's robustness with sensitivity testing. We run \textbf{DCS} with NAM to solve the SRx4 task on the FFHQ 256x256 dataset, varying $p_\text{critical}$ by a multiplicative factor and measure the LPIPS and PSNR scores. Results are displayed in Figure \ref{fig:pcrit_ablation}, and show that performance remains stable and consistent as long as $p_\text{critical}$ is large enough, and performs best around our chosen value (where the multiplier is $1$). This finding supports the conclusion that our optimization scheme with a statistically informed early stopping criterion helps prevent overfitting, as at lower values of $p_\text{critical}$, the performance starts to degrade. Note that this experiment does not detect underfitting as an underfit solution is still expected to produce a high fidelity image as it relies more on the unconditional diffusion model.

\subsection{Improvement on Jacobian-Free Implementations} 
\label{sec:jf_main}
A major gain in the empirical performance of \textbf{DCS} comes from the fact that it no longer requires backpropagations through the neural score function, which allows for reductions in both runtime and memory footprint. In theory, most diffusion-based solvers can be modified to remove this neural backpropagation step by applying a $\texttt{stop\_gradient}$ to the score function output (e.g., RED-Diff \citep{mardani2023variational}). We compare against RED-Diff in the main text, and additionally against ablated variants of DPS \citep{chung2022diffusion} and LGD \citep{song2023loss} in Appendix \ref{sec:jf}, and demonstrate clear improvements on these methods.

\section{IMPACT AND LIMITATIONS}

Our work proposes a substantial improvement to the framework of noise-robust guided generation via diffusion models. As shown in Section \ref{sec:experiments} and Appendix \ref{sec:additional_experiments}, \textbf{DCS} achieves consistent gains in performance and robustness, while also reducing runtime and memory requirements. These advances, along with the availability of our code, make the method particularly useful for researchers and practitioners with limited computational budgets, positioning \textbf{DCS} as a practical and readily adoptable tool for users working with inverse problems.

Our experiments focus on inverse problems with linear corrupting operators $\mathcal{A}$. While our theoretical framework extends to a broader class of non-linear operators (Appendix \ref{sec:proofs_suff}), further investigation is required to fully understand the behavior of \textbf{DCS} and competing methods in non-linear settings.

\section{CONCLUSION}
We proposed an effective adjustment to the diffusion-based inverse problem solver framework in the literature that improves speed and stability. Observing that the marginals of the diffusion process, which solve the inverse problem, are normally distributed at each time $t$, we derived a simple, single-parameter likelihood model, whose sole unknown is obtained via a tractable maximum likelihood estimation algorithm. We leveraged this new perspective to create a noise-aware maximization scheme, and demonstrated the effectiveness of our method with a suite of numerical experiments and qualitative comparisons.

\acknowledgments
This work was supported by NIH Grants U54AG076043, U54AG079759, R01DA063148, UM1DA051410, U01DA053628, and P50CA121974.

\bibliography{main}





\section*{Checklist}

\begin{enumerate}

  \item For all models and algorithms presented, check if you include:
  \begin{enumerate}
    \item A clear description of the mathematical setting, assumptions, algorithm, and/or model. [Yes]
    \item An analysis of the properties and complexity (time, space, sample size) of any algorithm. [Yes]
    \item (Optional) Anonymized source code, with specification of all dependencies, including external libraries. [Yes, see link in footnote]
  \end{enumerate}

  \item For any theoretical claim, check if you include:
  \begin{enumerate}
    \item Statements of the full set of assumptions of all theoretical results. [Yes]
    \item Complete proofs of all theoretical results. [Yes, proofs are in the supplementary document.]
    \item Clear explanations of any assumptions. [Yes]     
  \end{enumerate}

  \item For all figures and tables that present empirical results, check if you include:
  \begin{enumerate}
    \item The code, data, and instructions needed to reproduce the main experimental results (either in the supplemental material or as a URL). [Yes]
    \item All the training details (e.g., data splits, hyperparameters, how they were chosen). [Yes]
    \item A clear definition of the specific measure or statistics and error bars (e.g., with respect to the random seed after running experiments multiple times). [Yes]
    \item A description of the computing infrastructure used. (e.g., type of GPUs, internal cluster, or cloud provider). [Yes]
  \end{enumerate}

  \item If you are using existing assets (e.g., code, data, models) or curating/releasing new assets, check if you include:
  \begin{enumerate}
    \item Citations of the creator If your work uses existing assets. [Yes]
    \item The license information of the assets, if applicable. [Not Applicable]
    \item New assets either in the supplemental material or as a URL, if applicable. [Not Applicable]
    \item Information about consent from data providers/curators. [Not Applicable]
    \item Discussion of sensible content if applicable, e.g., personally identifiable information or offensive content. [Not Applicable]
  \end{enumerate}

  \item If you used crowdsourcing or conducted research with human subjects, check if you include:
  \begin{enumerate}
    \item The full text of instructions given to participants and screenshots. [Not Applicable]
    \item Descriptions of potential participant risks, with links to Institutional Review Board (IRB) approvals if applicable. [Not Applicable]
    \item The estimated hourly wage paid to participants and the total amount spent on participant compensation. [Not Applicable]
  \end{enumerate}

\end{enumerate}

\clearpage
\appendix
\thispagestyle{empty}

\onecolumn
\aistatstitle{Injecting Measurement Information Yields a Fast and Noise-Robust Diffusion-Based Inverse Problem Solver: \\
Supplementary Materials}

\section{BACKGROUND AND RELATED WORK}
\label{sec:background}
\subsection{Diffusion Models}
\label{sec:dms}
Inspired by non-equilibrium thermodynamics, denoising diffusion probabilistic models \citep{ho2020denoising} convert data $\mathbf{x}_0 \sim p_{\text{data}}(\mathbf{x})$ to noise $\mathbf{x}_T \sim \mathcal{N}(\mathbf{0}, \mathbf{I})$ via a diffusion process described by the variance-preserving stochastic differential equation (VP-SDE)
\begin{equation}
    d\mathbf{x} = -\frac{\beta(t)}{2} \mathbf{x} dt + \sqrt{\beta(t)} d{\mathbf{w}},
    \label{eq:ito_sde}
\end{equation}
where $\beta(t): \mathbb{R} \rightarrow [0, 1]$ is a monotonically increasing noise schedule and $\mathbf{w}$ is the standard Wiener process \citep{song2020score}. This leads to the marginal distribution
\begin{equation}
    p_t(\mathbf{x}_t) = \mathbb{E}_{\mathbf{x}_0 \sim p_{\text{data}}}  \big [ \mathcal{N}(\mathbf{x}_t; \sqrt{\alpha_t} \mathbf{x}_0, \; \underbrace{(1 - \alpha_t)}_{\sigma_t^2} \mathbf{I})  \big ], \hspace{.1in} \alpha_t = e^{-\frac{1}{2}\int_0^t \beta(s) ds},
    \label{eq:data_dependent_diffusion}
\end{equation}
where $\mathcal{N}(\ \cdot\ ; \mu, \Sigma)$ is the probability density function (pdf) of a normal distribution centered at $\mu$ with covariance $\Sigma$. Sampling from $p_{\text{data}}(\mathbf{x})$ can then occur by modeling the reverse diffusion, which has a simple form given by \citep{anderson1982reverse}
\begin{equation}
    d\overline{\mathbf{x}} = \left[-\frac{\beta(t)}{2} \mathbf{x} - \beta(t) \nabla_\mathbf{x} \log p_t(\mathbf{x}_t)\right]dt + \sqrt{\beta(t)} d\overline{\mathbf{w}},
    \label{eq:reverse_sde}
\end{equation}
with reverse-time Wiener process $\overline{\mathbf{w}}$ and score function $\nabla_\mathbf{x} \log p_t(\mathbf{x}_t)$. Therefore, diffusion model training consists of approximating the score function with a model
\begin{equation}
    s_\theta(\mathbf{x}_t, t) \approx \nabla_\mathbf{x} \log p_t(\mathbf{x}_t),
\end{equation}
and sampling consists of obtaining solutions to the reverse-time SDE (\ref{eq:reverse_sde}) with numerical solvers. A simple approach is given by the DDIM sampler with $\sigma_t = \sqrt{1 - \alpha_t}$ \citep{song2020denoising}
\begin{equation}
    \mathbf{x}_{t - 1} = \sqrt{\alpha_{t-1}} \frac{\mathbf{x}_t + \sigma^2_t \nabla \log p_t(\mathbf{x}_t)}{\sqrt{\alpha_t}} + \sigma_{t-1}\boldsymbol\epsilon,
    \label{eq:ddpm_sampling}
\end{equation}
where $\boldsymbol\epsilon \sim \mathcal{N}(0,\mathbf{I})$.

\subsection{Solving Inverse Problems with Diffusion Models}
\label{sec:tweedies}
When solving inverse problems with diffusion models, the aim is to leverage information from $\mathbf{y}$ to define a \textbf{modified} reverse diffusion process
\begin{equation}
    \mathbf{x}_T, \mathbf{x}_{T-1}, \dots, \mathbf{x}_1, \mathbf{x}_0,
\end{equation}
such that $\mathbf{x}_t$ coincides with the desired $\mathbf{x}$ (Eq. \ref{eq:intro_inverse_problem}) precisely at $t=0$. Previous approaches can generally be sorted into two categories, which we designate \textbf{posterior solvers} and \textbf{projection solvers}.
    
\paragraph{Posterior Solvers}
An intuitive approach is leveraging Bayes' rule to sample from the \textbf{posterior} distribution given a prior $p_t(\mathbf{x}_t)$ and observation $\mathbf{y}$:
\begin{equation}
    \mathbf{x}_t \sim p(\mathbf{x}_t|\mathbf{y}) = \frac{p(\mathbf{y} | \mathbf{x}_t) p(\mathbf{x}_t)}{p(\mathbf{y})}.
\end{equation}
Taking logs and gradients of both sides of the equation, we obtain a form of the conditional density that can be accurately approximated with the modeled score function
\begin{equation}
    \nabla \log p(\mathbf{x}_t|\mathbf{y}) = \nabla \log p(\mathbf{y} | \mathbf{x}_t) + \nabla \log p(\mathbf{x}_t) \approx \nabla \log p(\mathbf{y} | \mathbf{x}_t) + s_\theta(\mathbf{x}_t, t),
    \label{eq:posterior_bayes}
\end{equation}
and describes the core method of the DPS algorithm \citep{chung2022diffusion}. This strategy can also be extended to latent diffusion models, resulting in Latent-DPS and PSLD \citep{rout2023beyond}. Generally, the conditional term $\nabla \log p(\mathbf{y} | \mathbf{x}_t)$ cannot be exact due to reasons we will investigate subsequently in Section \ref{sec:tweedies_explained}, though these approximations are improved in LGD \citep{song2023loss} and STSL \citep{rout2024solving}. More recent work \citep{sun2024provable} propose an annealed Monte-Carlo-based perspective to posterior sampling, which results in a very similar algorithm to DPS. Much like MCG and ReSample (discussed in the next category), posterior solvers require estimating $\frac{\partial}{\partial \mathbf{x}_t} \mathbf{x}_0$ which involves backpropagation through the diffusion model, and significantly increases runtime and hampers scalability compared to unconditional sampling.

\paragraph{Projection Solvers} Another approach involves guiding the reverse diffusion process by directly \textbf{projecting} $\mathbf{x}_t$ onto a manifold $\mathcal{M} = \{\mathbf{x}: \mathcal{A}(\mathbf{x}) = y \} \subseteq \mathbf{R}^d$ at each time step, i.e.
\begin{align}
	\mathbf{x}_{t}' &= \mathbf{P}\hat{\mathbf{x}}_0 [\mathbf{x}_t] \\
    \mathbf{x}_{t - 1} &= \sqrt{\alpha_{t-1}} \frac{\mathbf{x}_t' + \sigma^2_t \nabla \log p(\mathbf{x}_t' | \hat{\mathbf{x}}_0 [\mathbf{x}_t])}{\sqrt{\alpha_t}} + \sigma_{t-1}\boldsymbol\epsilon.
\end{align}

Where $\hat{\mathbf{x}}_0[\mathbf{x}_t]$ is some prediction of $\mathbf{x}_0$ given only $\mathbf{x}_t$ (we elaborate in Section \ref{sec:tweedies_explained}), and $\mathbf{P}$ is either a projection onto the low rank subspace or range of $\mathcal{A}$. The resulting algorithms are DDRM \citep{kawar2022denoising} and DDNM \citep{wang2022zero}, respectively. Of course, this strategy is often restricted to situations where two conditions simultaneously hold true: (1) the measurement operator $\mathcal{A}$ is linear, and (2) the inverse problem is noiseless, i.e, $\eta$ is identically $\mathbf{0}$. These assumptions drastically limit the applicability of such models. The linearity restriction can be lifted by taking derivatives the measurement discrepancy, as in MCG \citep{chung2022improving} and ReSample \citep{song2024solving}, though this comes at the cost of significantly increased computation, requiring $\frac{\partial}{\partial \mathbf{x}_t} \mathbf{x}_0$ which involves backpropagating through the score network. Finally, \citep{cardoso2023monte} straddles the line between both categories --- while \texttt{MCGdiff} is ostensibly a Bayesian solver, it bears greater resemblance to projection solvers since it does not form the decomposition in Eq. \ref{eq:posterior_bayes} and also samples by projecting each iterate to the null-space of $\mathcal{A}$, thus implementing a projected n-particle sequential monte carlo (SMC) sampling algorithm.

\paragraph{A Maximum Likelihood Solver} We take a different perspective on solving the inverse problem. As seen in Section \ref{sec:tweedies_explained}, both \textbf{projection} and \textbf{posterior} solvers must quantify the discrepancy between $\mathbf{x}_t$ and $\mathbf{y}$ via the consistency error $||\mathcal{A}(\mathbf{x}_0) - \mathbf{y}||$ at each diffusion step. Due to the complexity of the diffusion process, this involves approximating a fundamentally intractable quantity. In Section \ref{sec:ours}, we construct a simpler process whose parameters can be obtained via maximum likelihood estimation. Unlike the evidence lower bound proposed in \citep{mardani2023variational}, we derive an explicit likelihood model, which is amenable to an optimization scheme with a probabilistic noise-aware stopping criterion. Finally, we show that the resulting algorithm is simple, fast, and adaptable to noise.

\section{ADDITIONAL THEOREMS AND PROOFS}
\label{sec:proofs}

\subsection{Proof of Tweedie's Formula}
For completeness, we include the statement and proof for Tweedie's formula.

\begin{theorem}[Tweedie's Formula]
Let $\mathbf{x}_0$ be a sample drawn from a distribution $p(\mathbf{x}_0)$. Then for any
\begin{equation}
    \mathbf{x}_t = \alpha_t \mathbf{x}_0 + \sigma_t \mathbf{z} \hspace{.5in} \mathbf{z} \sim \mathcal{N}(0, \mathbf{I})
\end{equation}
drawn from the marginal of the diffusion process on $p(\mathbf{x}_0)$ at time $t$, the \textbf{posterior mean} given $\mathbf{x}_t$ is
\begin{equation}
    \mathbb{E}[\mathbf{x}_0 | \mathbf{x}_t] = \frac{1}{\alpha_t} \left[\mathbf{x}_t + \sigma_t^2 \nabla \log p_t(\mathbf{x}_t)\right].
\end{equation}
\label{thm:tweedies}
\end{theorem}
\begin{proof}[Proof (of Lemma \ref{thm:tweedies})]
    Let $\phi_{\sigma}$ be the pdf of $\mathcal{N}(0,\sigma \mathbf{I})$. We first note that the marginal distribution at time $t$ can be written as
    \begin{equation}
    p_t(\mathbf{x}_t) = \left (p_{\alpha_t} \ast \phi_{\sigma_t} \right )(\mathbf{x}_t) = \int \phi_{\sigma_t}(\mathbf{x}_t - \mathbf{x}) p_{\alpha_t}(\mathbf{x}) d\mathbf{x},
    \end{equation}
    where 
    \begin{equation}
        p_{\alpha_t}(\mathbf{x}) = \frac{1}{\alpha_t} p\left(\alpha_t^{-1} \mathbf{x}\right)
    \end{equation}
    due to the probabilistic change-of-variables formula. Letting $\mathbf{x}' = \alpha_t \mathbf{x}_0$, we have the equality
    \begin{align*}
        \frac{\mathbb{E}[\mathbf{x}' | \mathbf{x}_t] - \mathbf{x}_t}{\sigma_t^2} 
        &= \int \frac{\mathbf{x}' - \mathbf{x}_t}{\sigma_t^2} p(\mathbf{x}' | \mathbf{x}_t) d\mathbf{x}' \\
        &= \int \frac{\mathbf{x}' - \mathbf{x}_t}{\sigma_t^2} \frac{p(\mathbf{x}', \mathbf{x}_t)}{p(\mathbf{x}_t)} d\mathbf{x}' \\
        &= \int \frac{\alpha_t \mathbf{x}_0 - \mathbf{x}_t}{\sigma_t^2} \frac{\phi_{\sigma_t}(\mathbf{x}_t - \mathbf{x}') p_{\alpha_t}(\mathbf{x}')}{\int \phi_{\sigma_t}(\mathbf{x}_t - \mathbf{x}) p_{\alpha_t}(\mathbf{x}) d\mathbf{x}} d\mathbf{x}' \\
        &= \int \left[\nabla_{\mathbf{x}_t} \phi_{\sigma_t}(\mathbf{x}_t - \mathbf{x}')\right] \frac{\phi_{\sigma_t}(\mathbf{x}_t - \mathbf{x}') p_{\alpha_t}(\mathbf{x}')}{\int \phi_{\sigma_t}(\mathbf{x}_t - \mathbf{x}) p_{\alpha_t}(\mathbf{x}) d\mathbf{x}} d\mathbf{x}' \\
        &= \nabla_{\mathbf{x}_t} \log \left[ \phi_{\sigma_t}(\mathbf{x}_t - \mathbf{x}') p_{\alpha_t}(\mathbf{x}')\right] \\
        &= \nabla \log p_t(\mathbf{x}_t).
    \end{align*}
    Re-arranging terms on either side of the equation, we obtain
    \begin{equation}
    \mathbb{E}[\mathbf{x}' | \mathbf{x}_t] = \mathbf{x}_t + \sigma_t^2 \nabla \log p_t(\mathbf{x}_t).
    \end{equation}
    Finally, we expand $\mathbf{x}' = \alpha_t \mathbf{x}_0$ and invoke the linearity of the expectation to arrive at
    \begin{equation}
    \mathbb{E}[\mathbf{x}_0 | \mathbf{x}_t] = \frac{1}{\alpha_t} \left[\mathbf{x}_t + \sigma_t^2 \nabla \log p_t(\mathbf{x}_t)\right].
    \end{equation}
    as desired.
\end{proof}

\subsection{Proof for Theorem \ref{thm:tweedie_iff}}

\texorpdfstring{$\mathbf{x}_t \sim \mathcal{N}(\bm{\mu}(\mathbf{x}_0), \sigma_t^2\mathbf{I}) \iff \mathbb{E}[\mathbf{x}_0 | \mathbf{x}_t] = \mathbf{x}_0$}{xt ~ N(mu(x0), sigma_t^2 I) iff E[x0 | xt] = x0}

We demonstrate sufficiency of the Gaussian-distributed condition by proving Lemma \ref{thm:tweedie}.
\begin{lemma}[Sufficent condition]
\label{thm:tweedie}
Let $\mathbf{x}_0$ be given. Suppose $\mathbf{x}_t$ is distributed as
\begin{equation}
	p_t(\mathbf{x}_t) = \mathcal{N}(\mathbf{x}_t; \sqrt{\alpha_t} \mathbf{x}_0, \underbrace{1 - \alpha_t}_{\sigma_t^2} \mathbf{I}).
    \label{eq:tweedie_distribution}
\end{equation}
Then $\mathbf{x}_0$ can be recovered via
\begin{equation}
	\mathbf{x}_0 = \frac{1}{\sqrt{\alpha_t}}\left[\mathbf{x}_t + \sigma_t^2 \nabla_{\mathbf{x}_t} \log p_t(\mathbf{x}_t)\right].
	\label{eq:x0_xt}
\end{equation}
\end{lemma}

\begin{proof}[Proof (of Lemma \ref{thm:tweedie})]
	\begin{align}
    \frac{1}{\sqrt{\alpha_t}}\left[\mathbf{x}_t + \sigma_t^2 \nabla_{\mathbf{x}_t} \log p_t(\mathbf{x}_t)\right] &= \frac{1}{\sqrt{\alpha_t}}\left[\mathbf{x}_t - \nabla_{\mathbf{x}_t} \sigma_t^2 \frac{1}{2 \sigma_t^2} || \mathbf{x}_t - \sqrt{\alpha_t}\mathbf{x}_0 ||_2^2 \right] \\
    &= \frac{1}{\sqrt{\alpha_t}}\left[\mathbf{x}_t - (\mathbf{x}_t - \sqrt{\alpha_t} \mathbf{x}_0) \right] \\
    &= \mathbf{x}_0.
\end{align}
\end{proof}

To demonstrate the necessary condition, we show that the inverse of Lemma \ref{thm:tweedie} also holds.

\begin{lemma}[Necessary condition]
    \label{thm:tweedie_sufficiency}
    If $\mathbf{x}_0$ can be recovered via Eq. \ref{eq:x0_xt}, then $ p_t(\mathbf{x}_t | \mathbf{x}_0)$ takes the form Eq. \ref{eq:tweedie_distribution}.
\end{lemma}

\begin{proof}[Proof (of Lemma \ref{thm:tweedie_sufficiency})]
    Suppose that
    \begin{equation}
        \mathbf{x}_0 = \frac{1}{\sqrt{\alpha_t}}\left[\mathbf{x}_t + \sigma_t^2 \nabla_{\mathbf{x}_t} \log p_t(\mathbf{x}_t)\right]
    \end{equation}
    Then we may re-arrange terms, obtaining
	\begin{align}
    \frac{\sqrt{\alpha_t} \mathbf{x}_0 - \mathbf{x}_t}{\sigma_t^2} = \nabla_{\mathbf{x}_t} \log p_t(\mathbf{x}_t).
    \end{align}
    Taking the anti-derivative of both sides, we conclude that
    \begin{equation}
         \log p_t(\mathbf{x}_t) = \frac{1}{2 \sigma_t^2} || \mathbf{x}_t - \sqrt{\alpha_t}\mathbf{x}_0 ||_2^2 + C.
    \end{equation}
    Since $\log p_t(\mathbf{x}_t | \mathbf{x}_0)$ can only take this form when $p_t(\mathbf{x}_t | \mathbf{x}_0)$ is a simple isotropic Gaussian distribution, we conclude our proof.
\end{proof}

\begin{proof}[Proof (of Theorem \ref{thm:tweedie_iff})]
    First, we note that Tweedie's formula \citep{efron2011tweedie} tells us that the posterior mean of a data distribution $\mathbf{x}_t \sim p_t(\mathbf{x}_t | \mathbf{x}_0)$ can be obtained via the relation
    \begin{equation}
        \mathbf{E}[\mathbf{x}_0 | \mathbf{x}_t] = \frac{1}{\sqrt{\alpha_t}}\left[\mathbf{x}_t + \sigma_t^2 \nabla_{\mathbf{x}_t} \log p_t(\mathbf{x}_t)\right].
    \end{equation}
    Then, since Lemmas \ref{thm:tweedie} and \ref{thm:tweedie_sufficiency} are converses of each other, we demonstrate that the conditions stated in Lemma \ref{thm:tweedie} are necessary and sufficient.
\end{proof}

\subsection{Proof for Theorem \ref{thm:tweedie_cond_iff}} 

\texorpdfstring{$\mathbf{x}_t \sim \mathcal{N}(\bm{\mu}(\mathbf{x}_0, \mathbf{y}), \sigma_t^2\mathbf{I}) \iff \mathbb{E}[\mathbf{x}_0 | \mathbf{x}_t, \mathbf{y}] = \mathbf{x}_0$}{xt ~ N(mu(x0, y), sigma_t^2 I) iff E[x0 | xt, y] = x0}

First, we show that the measurement consistent Tweedie's formula holds for a given diffusion variate $\mathbf{x}_t$ and measurement $\mathbf{y}$ .
\begin{theorem}[Conditional Tweedie's Formula]
Let $\mathbf{x}_0$ be a sample drawn from a conditional distribution $p(\mathbf{x}_0 | \mathbf{y})$. Then for any
\begin{equation}
    \mathbf{x}_t = \alpha_t \mathbf{x}_0 + \sigma_t \mathbf{z} \hspace{.5in} \mathbf{z} \sim \mathcal{N}(0, \mathbf{I})
\end{equation}
drawn from the marginal of the diffusion process on $p(\mathbf{x}_0 | \mathbf{y})$ at time $t$, the \textbf{conditional posterior mean} given $\mathbf{x}_t$ is
\begin{equation}
    \mathbb{E}[\mathbf{x}_0 | \mathbf{x}_t, \mathbf{y}] = \frac{1}{\alpha_t} \left[\mathbf{x}_t + \sigma_t^2 \nabla \log p_t(\mathbf{x}_t | \mathbf{y})\right].
\end{equation}
\label{thm:cond_tweedies}
\end{theorem}
\begin{proof}[Proof (of Lemma \ref{thm:cond_tweedies})]
    Let $\phi_{\sigma}$ be the pdf of $\mathcal{N}(0,\sigma \mathbf{I})$. We first note that the marginal distribution at time $t$ can be written as
    \begin{equation}
    p_t(\mathbf{x}_t | \mathbf{y}) = \left (p_{\alpha_t}( \cdot | \mathbf{y}) \ast \phi_{\sigma_t} \right )(\mathbf{x}_t) = \int \phi_{\sigma_t}(\mathbf{x}_t - \mathbf{x}) p_{\alpha_t}(\mathbf{x} | \mathbf{y}) d\mathbf{x},
    \end{equation}
    where 
    \begin{equation}
        p_{\alpha_t}(\mathbf{x} | \mathbf{y}) = \frac{1}{\alpha_t} p\left(\alpha_t^{-1} \mathbf{x} | \mathbf{y}\right)
    \end{equation}
    due to the probabilistic change-of-variables formula. Letting $\mathbf{x}' = \alpha_t \mathbf{x}_0$, we have the equality
    \begin{align*}
        \frac{\mathbb{E}[\mathbf{x}' | \mathbf{x}_t, \mathbf{y}] - \mathbf{x}_t}{\sigma_t^2} 
        &= \int \frac{\mathbf{x}' - \mathbf{x}_t}{\sigma_t^2} p(\mathbf{x}' | \mathbf{x}_t, \mathbf{y}) d\mathbf{x}' \\
        &= \int \frac{\mathbf{x}' - \mathbf{x}_t}{\sigma_t^2} \frac{p(\mathbf{x}', \mathbf{x}_t| \mathbf{y})}{p(\mathbf{x}_t | \mathbf{y})} d\mathbf{x}'. \\
    \end{align*}
    We now substitute $p(\mathbf{x}',\mathbf{x}_t|\mathbf{y}) = p(\mathbf{x}' - \mathbf{x}_t | \mathbf{y})p_{\alpha_t}(\mathbf{x}' | \mathbf{y}) = \phi_{\sigma_t} (\mathbf{x}' - \mathbf{x}_t) p_{\alpha_t}(\mathbf{x}' | \mathbf{y})$, due to the fact that the quantity $\mathbf{x}' - \mathbf{x}_t \sim \mathcal{N}(0, \sigma \mathbf{I})$ is independent of $\mathbf{x}'$ and $\mathbf{y}$. Hence, we continue the derivation
    \begin{align*}
        \int \frac{\mathbf{x}' - \mathbf{x}_t}{\sigma_t^2} \frac{p(\mathbf{x}', \mathbf{x}_t| \mathbf{y})}{p(\mathbf{x}_t | \mathbf{y})} d\mathbf{x}' &= \int \frac{\alpha_t \mathbf{x}_0 - \mathbf{x}_t}{\sigma_t^2} \frac{\phi_{\sigma_t}(\mathbf{x}_t - \mathbf{x}') p_{\alpha_t}(\mathbf{x}' | \mathbf{y})}{\int \phi_{\sigma_t}(\mathbf{x}_t - \mathbf{x}) p_{\alpha_t}(\mathbf{x} | \mathbf{y}) d\mathbf{x}} d\mathbf{x}' \\
        &= \int \left[\nabla_{\mathbf{x}_t} \phi_{\sigma_t}(\mathbf{x}_t - \mathbf{x}')\right] \frac{\phi_{\sigma_t}(\mathbf{x}_t - \mathbf{x}') p_{\alpha_t}(\mathbf{x}' | \mathbf{y})}{\int \phi_{\sigma_t}(\mathbf{x}_t - \mathbf{x}) p_{\alpha_t}(\mathbf{x} | \mathbf{y}) d\mathbf{x}} d\mathbf{x}' \\
        &= \nabla_{\mathbf{x}_t} \log \left[ \phi_{\sigma_t}(\mathbf{x}_t - \mathbf{x}') p_{\alpha_t}(\mathbf{x}' | \mathbf{y})\right] \\
        &= \nabla \log p_t(\mathbf{x}_t | \mathbf{y}).
    \end{align*}
    Re-arranging terms on either side of the equation, we obtain
    \begin{equation}
    \mathbb{E}[\mathbf{x}' | \mathbf{x}_t, \mathbf{y}] = \mathbf{x}_t + \sigma_t^2 \nabla \log p_t(\mathbf{x}_t | \mathbf{y}).
    \end{equation}
    Finally, we expand $\mathbf{x}' = \alpha_t \mathbf{x}_0$ and invoke the linearity of the expectation to arrive at
    \begin{equation}
    \mathbb{E}[\mathbf{x}_0 | \mathbf{x}_t, \mathbf{y}] = \frac{1}{\alpha_t} \left[\mathbf{x}_t + \sigma_t^2 \nabla \log p_t(\mathbf{x}_t | \mathbf{y})\right].
    \end{equation}
    as desired.
\end{proof}

Now, the main theorem follows.

\begin{proof}[Proof (of Theorem \ref{thm:tweedie_iff})]
    We note that the proofs for Lemmas \ref{thm:tweedie} and \ref{thm:tweedie_sufficiency} remain the same if we let $p_t(\mathbf{x}_t|\mathbf{x}_0) = p_t(\mathbf{x}_t|\mathbf{x}_0, \mathbf{y})$. Thus we again observe that the proofs for Lemmas \ref{thm:tweedie} and \ref{thm:tweedie_sufficiency} are converses of each other, and demonstrate that the conditions stated in Lemma \ref{thm:tweedie} are necessary and sufficient.
\end{proof}

\subsection{Theorems for Sufficiency} \label{sec:proofs_suff}

We set up Theorems to show that the estimator in Eq. \ref{eq:data_conditional_correction} is a sufficient statistic under different properties of $\mathcal{A}$. Letting $\mathbf{f}(\mathbf{y})$ be the function that obtains $\nabla \log p_t(\mathbf{x}_t | \mathbf{x}_0)$ via Eq. \ref{eq:data_conditional_correction}, we show that $\mathbf{y}$ is measurable under the sigma algebra induced by the measurement $\mathbf{f}$. 

Intuitively, we demonstrate that $\mathbf{f}(\mathbf{y})$ contains as much information as possible about the underlying signal $\mathbf{x}_0$ as can be gathered via $\mathbf{y}$. The theoretical and intuitive statements can be summarized by the simple conditional equivalence
\begin{equation}
	p(\mathbf{y} | {\epsilon_\mathbf{y}}_*, \mathbf{x}_0) = p(\mathbf{y} | {\epsilon_\mathbf{y}}_*).
\end{equation}
In Theorem \ref{thm:sufficiency}, we consider two simple and theoretically similar cases: when $\mathbf{y} = \mathcal{A}(\mathbf{x})$ is noise-free, and when $\mathcal{A}$ is linear. We restate it here in a less condensed form for clarity:\\
\textbf{Theorem \ref{thm:sufficiency}}. Let $\mathbf{y} = \mathcal{A}(\mathbf{x}_0) + \boldsymbol{\eta}$ be an observation from the forward measurement model, such that either $\boldsymbol{\eta} = 0$, or $\mathcal{A}$ is linear. If
\begin{equation}
	{\epsilon_\mathbf{y}}_* = \underset{\epsilon_\mathbf{y}}{\text{argmax}}\quad \log p\left(\mathbf{y} \bigg| \frac{1}{\sqrt{\alpha_t}}(\mathbf{x}_t + \sigma^2_t{\epsilon_\mathbf{y}})\right),
\end{equation}
then
\begin{equation}
    p(\mathbf{y} | {\epsilon_\mathbf{y}}_*, \mathbf{x}_0) = p(\mathbf{y} | {\epsilon_\mathbf{y}}_*).
\end{equation}



While we do not evaluate the empirical performance of \textbf{DCS} on non-linear operators, we additionally investigate the general noisy case where $\mathcal{A}$ is allowed to be nonlinear. We find that our results can still be quite general: we only need to assume $\mathcal{A}$ surjective, meaning that there exists some $\mathbf{x} \in \text{domain}(\mathcal{A})$ such that $\mathcal{A}(\mathbf{x}) = \mathbf{y}$. In fact, this result is slightly stronger --- we are able to show that sufficiency holds for $\mathcal{A}$ that are compositions of linear and surjective functions.

\begin{theorem}\label{thm:general_sufficiency}
    Let ${\epsilon_\mathbf{y}}_*$ be as defined in Theorem \ref{thm:sufficiency}. Suppose the twice-differentiable operator $\mathcal{A}:= \mathbf{P}^T \circ \mathbf{\phi}$ is composed of $\mathbf{P}: \mathbb{R}^d \rightarrow \mathbb{R}^r$, a linear projection, and $\mathbf{\phi}: \mathbb{R}^n \rightarrow \mathbb{R}^r$, an arbitrary surjective function. We have that
    \begin{equation}
        p(\mathbf{y} | {\epsilon_\mathbf{y}}_*, \mathbf{x_0}) = p(\mathbf{y} | {\epsilon_\mathbf{y}}_*).
    \end{equation}
\end{theorem}

To prove Theorems \ref{thm:sufficiency} and \ref{thm:general_sufficiency}, we establish the following Lemma which characterizes useful information about $\mathbf{x}_0^*$.

\begin{lemma} 
    \label{lem:optimality_construction}
    Suppose $\mathbf{y} \in \mathbb{R}^k$ is fixed, $\mathbf{x}_t \in \mathbb{R}^n$, with twice differentiable operator $\mathcal{A}: \mathbb{R}^n \rightarrow \mathbb{R}^k$. Then, for ${\epsilon_\mathbf{y}} = \nabla_{\mathbf{x}_t} \log p_t(\mathbf{x}_t | \mathbf{x}_0)$ which maximizes $p(\mathbf{y}|\mathbf{x}_0)$, the following holds true:
    \begin{enumerate}
        \item if $\eta = 0$ (i.e. the noiseless regime), $\mathcal{A}(\mathbf{x}_0) = \mathcal{A}(\mathbf{x}_t + \sigma_t^2 {\epsilon_\mathbf{y}}^*)$
        \item if $\mathcal{A}$ is surjective, $\mathcal{A}(\mathbf{x}_0) = \mathcal{A}(\mathbf{x}_t + \sigma_t^2 {\epsilon_\mathbf{y}}^*)$
        \item if $\mathcal{A}$ is linear, $\langle \mathbf{y} - \mathcal{A}(\mathbf{x}_t + \sigma_t^2 {\epsilon_\mathbf{y}}^*), \mathcal{A}(\mathbf{x}_t + \sigma_t^2 {\epsilon_\mathbf{y}}^*) - \mathcal{A}(\mathbf{x}_0) \rangle = 0$.
    \end{enumerate}

\end{lemma}

An interpretation of statement 3 reads that the optimal solution ${\epsilon_\mathbf{y}}^*$ for estimating $\mathbf{x}_0$ is orthogonal to the error to $\mathbf{y}$ in the linear case. The requirements of statement 3 may be relaxed to the statement $\mathcal{A}(\mathbf{x}) - \mathcal{A}(\mathbf{z})$ is in the range of the Jacobian of $\mathcal{A}$ at $\mathbf{z}$, however this is less intuitive than linearity. We avoid invoking linearity of $\mathcal{A}$ as long as possible to illustrate the fact that other transformations may share this property as well.

\begin{proof}[Proof (of Lemma \ref{lem:optimality_construction})]

    We will make use of the bijective mapping $\mathbf{z} \mapsto \mathbf{x}_t + \sigma_t^2 {\epsilon_\mathbf{y}}$, and charactarize the minima which maximize $\log p(\mathbf{y} | \mathbf{x}_0)$. We can solve the optimization problem,

    \[ \underset{\mathbf{z}}{\text{argmin}}\ \ || \mathbf{y} - \mathcal{A}(\mathbf{z}) ||^2_2 \]
    
    A minima to this objective can be characterized by the first order necessary condition,
    \begin{align*}
        \nabla_{\mathbf{z}} || \mathbf{y} - \mathcal{A} (\mathbf{z}) ||_2^2 &= -2\mathbf{J}_\mathbf{z}[\mathcal{A}](\mathbf{z})^T (\mathbf{y} - \mathcal{A}(\mathbf{z}))\\
        &= -2\mathbf{J}_\mathbf{z}[\mathcal{A}](\mathbf{z})^T (\mathcal{A}(\mathbf{x}_0) - \eta - \mathcal{A}(\mathbf{z})) := 0.
    \end{align*}
    We can confirm it is a minima by checking the solution of the above with,

    \begin{align*}
        \mathbf{H}_{\mathbf{z}} \left [ || \mathbf{y} - \mathcal{A} (\mathbf{z}) ||_2^2 \right ](\mathbf{z}^*) &= 2\nabla_\mathbf{z} \left [ \mathbf{J}_\mathbf{z}[\mathcal{A}](\mathbf{z})^T (\mathcal{A}(\mathbf{x}_0) + \eta) \right ] (\mathbf{z}^*)\\
        &= 2 \mathbf{J}_\mathbf{z}[\mathcal{A}](\mathbf{z}^*)^T \mathbf{J}_\mathbf{z}[\mathcal{A}](\mathbf{z}^*) + \sum_{j=1}^k \mathbf{H}_\mathbf{z}[\mathcal{A}_{(j)}](\mathbf{z}^*) \left ( \mathbf{y} - \mathcal{A} (\mathbf{z}^*) \right )\\
        &\succcurlyeq 0.\\
    \end{align*}

    If $\eta = 0$, we have that $\mathcal{A}(\mathbf{x}_0) = \mathbf{y}$, and therefore choosing any $\mathcal{A}(\mathbf{z}^*) = \mathcal{A}(\mathbf{x}_0)$ satisfies the first order condition. The second order condition is furthermore satisfied, as $\mathbf{y} - \mathcal{A} (\mathbf{z}^*) = 0$, meaning,

    \begin{align*}
        \mathbf{H}_{\mathbf{z}} \left [ || \mathbf{y} - \mathcal{A} (\mathbf{z}) ||_2^2 \right ](\mathbf{z}^*) &= 2 \mathbf{J}_\mathbf{z}[\mathcal{A}](\mathbf{z}^*)^T \mathbf{J}_\mathbf{z}[\mathcal{A}](\mathbf{z}^*) \succcurlyeq 0.\\
    \end{align*}

    This satisfies statement 1. Statement 2 is satisfied similarly, by choosing the same $\mathbf{z}$. Note that this case differs, in that $\mathbf{z} = \mathbf{x}_0$ is no longer necessarily a valid solution.

    Statement 3, is already satisfied in the cases where $\mathcal{A}$ has rank equal to the dimension of its co-domain (if $d = n$, this is equivalent to being full rank), since $\mathbf{y} - \mathcal{A}(\mathbf{z}^*) = 0$. Therefore we assume $\mathcal{A}$ is low-rank to prove the remaining cases.

    To show orthogonality between $\mathbf{y} - \mathcal{A}(\mathbf{z}^*)$ and $\mathcal{A}(\mathbf{z}^*) - \mathcal{A}(\mathbf{x}_0)$ in other cases, we let $\eta = \delta + \delta_\perp$. We can choose an optimal value for $\delta_\perp$ that satisfies $\delta_\perp^* = \underset{\mathbf{\delta_\perp}}{\inf} \left \{ || \mathbf{y} - \mathcal{A}(\mathbf{z}^*) - \delta_\perp ||^2_2 \right \}$, for the optimal value, $\mathbf{z}^*$. Due to the non-negativity and $0$ preserving properties of norms, we have,
    
    \begin{align*}
        \delta_\perp^* &=  \mathbf{y} - \mathcal{A}(\mathbf{z}^*)\\
        &=  \mathcal{A}(\mathbf{x}_0) + \eta - \mathcal{A}(\mathbf{z}^*)\\
        &= \mathcal{A}(\mathbf{x}_0) + \delta_\perp^* + \delta^* - \mathcal{A}(\mathbf{z}^*)\\
        \implies \delta^* &= \mathcal{A}(\mathbf{z}^*) - \mathcal{A}(\mathbf{x}_0).
    \end{align*}

    At the optima of the original objective, $\mathbf{z}^*$, the first order necessary condition dictates that,
    
    \begin{align*}
        \mathbf{J}_\mathbf{z}[\mathcal{A}](\mathbf{z}^*)^T (y - \mathcal{A}(\mathbf{z}^*)) &= \mathbf{J}_\mathbf{z}[\mathcal{A}](\mathbf{z}^*)^T \delta_\perp^* := 0.\\
    \end{align*}

    For a linear $\mathcal{A}$, the Jacobian is constant, so let $\mathbf{J}_\mathbf{z}[\mathcal{A}] = \mathbf{J}$. Therefore, $\mathbf{J}^T \delta_\perp^* = 0$, meaning $\delta_\perp^* \in \mathcal{N}(\mathbf{J}^T)$.\\

    Simultaneously, since $\delta^* = \mathcal{A}(\mathbf{x}) - \mathcal{A}(\mathbf{z}^*) = \mathcal{A}(\mathbf{x}-\mathbf{z}^*) = \mathbf{J} (\mathbf{x} - \mathbf{z}^* )$, we have $\delta^* \in \mathcal{R}(\mathbf{J}^T)$. Therefore due to the orthogonality of range and null spaces of a matrix, $\langle \delta_\perp^* ,\, \delta^* \rangle = 0$, completing the proof.

\end{proof}

We are now able to prove the theorems in the main text.
\begin{proof}[Proof of Theorem \ref{thm:sufficiency}]
    We leverage the theory of sufficient statistics to demonstrate our result. Namely, if ${\epsilon_\mathbf{y}}_*$ is a sufficient statistic for $\mathbf{y}$, then,
\begin{equation}
	p(\mathbf{y} | {\epsilon_\mathbf{y}}_*) = p(\mathbf{y} | {\epsilon_\mathbf{y}}_*, \mathbf{x}_0).
\end{equation}
    Therefore it suffices to demonstrate that ${\epsilon_\mathbf{y}}_*$ is a sufficient statistic for $\mathbf{y}$.
    
	By the Neyman-Fisher Factorization theorem, we have that a necessary and sufficient condition is if there exists non-negative functions $g$ and $h$ such that
	\begin{equation}
		p(\mathbf{y} | \mathbf{x}_0) = g({\epsilon_\mathbf{y}}_*, \mathbf{x}_0) h(\mathbf{y}).
	\end{equation}
	
	We observe that since $\eta \sim \mathcal{N}(\mathbf{0}, \sigma_\mathbf{y}^2 \mathbf{I})$, our random variable $\mathbf{y}$ can be characterized by the density function
	\begin{equation}
		p(\mathbf{y} | \mathbf{x}_0) = \mathcal{N}(\mathbf{y}; \mu = \mathcal{A}(\mathbf{x}_0), \Sigma = \sigma_\mathbf{y}^2 \mathbf{I}).
	\end{equation}
	Therefore, letting $\mathbf{y}_{{\epsilon_\mathbf{y}}_*} = \mathcal{A}(\frac{1}{\sqrt{\alpha_t}}(\mathbf{x}_t + \sigma^2_t{\epsilon_\mathbf{y}}_*))$, we can write
	\begin{align}
		p(\mathbf{y} | \mathbf{x}_0) 
		&= (2 \pi \sigma_\mathbf{y}^2)^{-n / 2} \exp\left(-\frac{1}{2 \sigma_\mathbf{y}^2} || \mathbf{y} - \mathcal{A}(\mathbf{x}_0) ||_2^2 \right) \\
		&= (2 \pi \sigma_\mathbf{y}^2)^{-n / 2} \exp\left(-\frac{1}{2 \sigma_\mathbf{y}^2} \left(|| \mathbf{y} - \mathbf{y}_{{\epsilon_\mathbf{y}}_*}||_2^2 + ||\mathbf{y}_{{\epsilon_\mathbf{y}}_*} - \mathcal{A}(\mathbf{x}_0)||_2^2 + 2\langle \mathbf{y} - \mathbf{y}_{{\epsilon_\mathbf{y}}_*}, \mathbf{y}_{{\epsilon_\mathbf{y}}_*} - \mathcal{A}(\mathbf{x}_0) \rangle \right) \right)\\
		&= (2 \pi \sigma_\mathbf{y}^2)^{-n / 2} \exp\left(-\frac{1}{2 \sigma_\mathbf{y}^2} ||\mathbf{y}_{{\epsilon_\mathbf{y}}_*} - \mathcal{A}(\mathbf{x}_0)||_2^2 \right) \exp\left(-\frac{1}{2 \sigma_\mathbf{y}^2} || \mathbf{y} - \mathbf{y}_{{\epsilon_\mathbf{y}}_*}||_2^2 \right),
	\end{align}
	where the third equality is due to Lemma \ref{lem:optimality_construction}. In the case that $\mathcal{A}$ is surjective, or the noiseless regime, statements 2 and 1 respectively satisfy the equality above trivially, as $\mathbf{y} = \mathbf{y}_{{\epsilon_\mathbf{y}}_*}$. If the operator is otherwise linear, statement 3 shows the cross term vanishes.
	
	Therefore, we can assign
	\begin{align}
		g({\epsilon_\mathbf{y}}_*, \mathbf{x}_0) &= (2 \pi \sigma_\mathbf{y}^2)^{-n / 2} \exp\left(\frac{1}{2 \sigma_\mathbf{y}^2} ||\mathbf{y}_{{\epsilon_\mathbf{y}}_*} - \mathcal{A}(\mathbf{x}_0)||_2^2 \right)\\
        h(\mathbf{y}) &= \exp\left(\frac{1}{2 \sigma_\mathbf{y}^2} || \mathbf{y} - \mathbf{y}_{{\epsilon_\mathbf{y}}_*}||_2^2 \right).
	\end{align}
    In the case where the measurement process $\mathcal{A}(\mathbf{x}) = \mathbf{y}$ is noiseless, this implies $h(\mathbf{y}) = 1$.
\end{proof}

We now modify the argument in order to relax the linearity assumption.

\begin{proof}[Proof of Theorem \ref{thm:general_sufficiency}]

    Let $\mathbf{z} = \frac{1}{\sqrt{\alpha_t}} \left ( \mathbf{x}_t + \sigma_t^2 {\epsilon_\mathbf{y}} \right )$, and $\mathbf{z}^* = \underset{\mathbf{z}}{\text{argmin}} \left \{ || \mathbf{y} - \mathcal{A}(\mathbf{z}) || \right \}$.

    Since $\mathbf{z}^*$ minimizes the objective $||\mathbf{y} - \mathcal{A}(\mathbf{z})||$, we also have that,
    
    $$\phi(\mathbf{z}^*) := \underset{\alpha}{\text{argmin}} \left \{ ||\mathbf{y} - \mathbf{P}^T(\alpha)|| \right \} = \underset{\alpha}{\text{argmax}}\ \  p(\mathbf{y} | \alpha).$$
    
    We can invoke Lemma \ref{lem:optimality_construction} to say

    \begin{align*}
        || \mathbf{y} - \mathbf{P}^T \phi(\mathbf{x}_0) ||_2^2 &= || \mathbf{y} - \mathbf{P}^T  \phi(\mathbf{z}^*)||_2^2  + ||\mathbf{P}^T  \phi(\mathbf{z}^*) - \mathbf{P}^T \phi (\mathbf{x}_0)||_2^2,\\
    \end{align*}

    since $\mathbf{P}^T$ is a linear operator, and $\phi(\mathbf{z}^*)$ satisfies the conditions in the lemma. Therefore, we have,

    \begin{align*}
        p(\mathbf{y} | \mathbf{x}_0) &= (2 \pi \sigma_\mathbf{y}^2)^{-n / 2} \exp\left(-\frac{1}{2 \sigma_\mathbf{y}^2} || \mathbf{y} - \mathbf{P}^T\phi(\mathbf{x}_0) ||_2^2 \right) \\
        &= (2 \pi \sigma_\mathbf{y}^2)^{-n / 2} \exp\left(-\frac{1}{2 \sigma_\mathbf{y}^2} || \mathbf{y} - \mathbf{P}^T \phi(\mathbf{z}^*)||_2^2 \right ) \exp\left(-\frac{1}{2 \sigma_\mathbf{y}^2} ||\mathbf{P}^T \phi (\mathbf{z}^*) - \mathbf{P}^T \phi (\mathbf{x}_0)||_2^2 \right ).\\
	\end{align*}

    We assign terms,

    \begin{align}
		g(\mathbf{z}_*, \mathbf{x}_0) &=  (2 \pi \sigma_\mathbf{y}^2)^{-n / 2} \exp\left(-\frac{1}{2 \sigma_\mathbf{y}^2} ||\mathbf{P}^T \phi (\mathbf{z}^*) - \mathbf{P}^T \phi (\mathbf{x}_0)||_2^2 \right ) \\
        h(\mathbf{y}) &= \exp\left(-\frac{1}{2 \sigma_\mathbf{y}^2} || \mathbf{y} - \mathbf{P}^T \phi(\mathbf{z}^*)||_2^2 \right ),\\
	\end{align}

    and once again invoke the Neyman-Fisher Factorization theorem to show $\mathbf{z}^*$ is sufficient for $\mathbf{y}$. Since ${\epsilon_\mathbf{y}}_*$ is a bijective mapping from $\mathbf{z}^*$, we have that ${\epsilon_\mathbf{y}}_*$ is sufficient, and similarly to Theorem \ref{thm:sufficiency} we state, $p(\mathbf{y}|{\epsilon_\mathbf{y}}_*) = p(\mathbf{y} | {\epsilon_\mathbf{y}}_*, \mathbf{x}_0) = p(\mathbf{y} | {\epsilon_\mathbf{y}}_*)$.
\end{proof}

Finally, we note that this proof provides necessary conditions, but not sufficient conditions for the sufficiency of $\textbf{DCS}$'s estimator. In this work, we do not investigate operators outside of the scope of Theorem \ref{thm:general_sufficiency}, there are potentially even weaker conditions on $\mathcal{A}$ that exist.

\section{INVERTABILITY OF FORWARD OPERATOR}
\label{sec:noninvertible_a}
Often, $\mathcal{A}$ is simply non-invertible (e.g. for super-resolution, inpainting, phase retrieval, and sparse MRI reconstruction tasks). With other tasks such as signal deblurring, the invertibility of $\mathcal{A}$ is often mathematically possible, but not numerically stable. In theory, blurring operator can be represented as convolution operators on the signal $\mathbf{x}$. Theorefore, the convolution theorem tells us that inverting a blurring operator $\mathcal{G}(\ast)$ on $\mathbf{x}$ is as simple as taking the quotient of the convolved signal $\mathbf{y} = \mathcal{A}(\mathbf{x})$ against the convolution kernel in the frequency domain, i.e.,
\begin{equation}
    \mathbf{x} = \mathcal{F}^{-1}[\mathcal{F}(\mathbf{y}) / \mathcal{F}(\mathcal{G})] = \mathbf{y} \ast \mathcal{F}^{-1}[\mathcal{F}(\mathcal{G})^{-1}]
    \label{eq:noninvertible_a}
\end{equation}
where $\mathcal{F}$ denotes the Fourier operator. However, in practice, there are implicit assumptions in Eq. \ref{eq:noninvertible_a}, such as the computability of $\mathcal{F}(\mathcal{G})$ and the existence of $\mathcal{F}(\mathcal{G})^{-1}$, that may not always hold. In particular, blur kernels are often truncated in practice, resulting in a highly ill-conditioned $\mathcal{F}(\mathcal{G})$ in the frequency domain, and numerical unstable (or non-existent) inverses. Ultimately, directly inverting $\mathcal{A}$ often fails to produce the highest quality results, even though it may be possible.

\section{ADDITIONAL EXPERIMENTS} \label{sec:additional_experiments}
In this section, we provide further comparisons against latent and Jacobian-free methods (Table \ref{table:inverse_solvers_overview_latent_jf}).

\subsection{Comparison against Latent Models (Table \ref{table:main_quantitative_latent})}
We show that our pixel-based model also performs favorably against latent models in Table \ref{table:main_quantitative_latent}. We retain the same experimental setting on pixel-based models as in Table \ref{table:main_quantitative}. For FFHQ, we use the pretrained FFHQ model weights from \citep{chung2022diffusion} for our method, and the pretrained FFHQ model with a VQ-F4 first stage model \citep{rombach2022high} in latent space models. For ImageNet, we again use pretrained model weights from \citep{chung2022diffusion} in pixel-based diffusion solvers, and the Stable Diffusion v1.5 latent model for latent solvers. As with pixel-based methods many existing works suffer in the presence of additional noise. Further implementation details are discussed in Appendix \ref{sec:implementation_details}.

\begin{figure}
    \captionof{table}{Description of latent and Jacobian-free solvers used for comparisons in text. For each solver we list the type (as described in Section \ref{sec:tweedies}), optimization space (pixel or latent), whether it requires backpropagation through a neural function evaluation (NFE, i.e., the score network call), as well as runtime and memory footprint.}
    \label{table:inverse_solvers_overview_latent_jf}
    \centering
        \begin{tabular}{l|lcccc}
            \textbf{Solver} & \textbf{Type} & \textbf{Space} & \begin{tabular}{@{}c@{}}\textbf{No NFE}  \\ \textbf{Backprop} \end{tabular}  & \textbf{Runtime} & \textbf{Memory}\\
            \hline
            \textbf{DCS} (Ours) & Hybrid & Pixel & \cmark & \textbf{1x} & \textbf{1x}\\
            \midrule
            Latent-DPS \citep{chung2023diffusion} \footnotemark[3] &  Posterior & Latent & \xmark &$6.1\mathbf{x}$ & $8.9\mathbf{x}$\\
            PSLD \citep{rout2023beyond} & Posterior & Latent & \xmark &$7.5\mathbf{x}$ & $15\mathbf{x}$\\
            STSL \citep{rout2024solving} & Posterior & Latent & \xmark &$1.85\mathbf{x}$ & $9\mathbf{x}$ \\
            ReSample \citep{song2024solving} & Projection & Latent & \cmark\footnotemark[4] &$29.5\mathbf{x}$ & $8.95\mathbf{x}$ \\
            \midrule
            DPS-JF \citep{chung2023diffusion} & Posterior & Pixel & \cmark & $1.5\mathbf{x}$ & $1.1\mathbf{x}$\\
            LGD-MC (n=10) \cite{song2023loss} & Posterior & Pixel & \xmark & 6x & 3.2x \\
            LGD-MC-JF (n=10) \cite{song2023loss} & Posterior & Pixel & \xmark & $2\mathbf{x}$ & $1.1\mathbf{x}$ \\
            \bottomrule
        \end{tabular}
\end{figure}

\footnotetext[3]{Latent-DPS is a direct application of DPS \cite{chung2023diffusion} to latent diffusion models. It is also mentioned in \citep{rout2023beyond}.}
\footnotetext[4]{As described in \citep{song2024solving}, ReSample does not run backpropagation on the score network, however the implementation does (Appendix \ref{sec:appendix_resample}).}

\begin{table*}
    \caption{Quantitative comparison against latent models on FFHQ 256x256-1K and ImageNet-1K datasets across various inverse problem tasks and noise levels ($\sigma_{\mathbf{y}} \in \{ 0.01, 0.1 \}$).}
    \label{table:main_quantitative_latent}
    \centering
    \begin{adjustbox}{max width=\linewidth}
        \begin{tabular}{@{}lccccccccccccccc@{}}
            \toprule
            \textbf{FFHQ} & \multicolumn{3}{c}{SR $\times 4$} & \multicolumn{3}{c}{Random Inpainting} & \multicolumn{3}{c}{Box Inpainting} & \multicolumn{3}{c}{Gaussian Deblurring} & \multicolumn{3}{c}{Motion Deblurring}\\
            \cmidrule(r){2-4}\cmidrule(lr){5-7}\cmidrule(l){8-10}\cmidrule(l){11-13}\cmidrule(l){14-16}
            $\sigma_\mathbf{y}=0.01$ & LPIPS $\downarrow$ & PSNR $\uparrow$ & FID $\downarrow$ & LPIPS $\downarrow$ & PSNR $\uparrow$ & FID $\downarrow$ & LPIPS $\downarrow$ & PSNR $\uparrow$ & FID $\downarrow$ & LPIPS $\downarrow$ & PSNR $\uparrow$ & FID $\downarrow$ & LPIPS $\downarrow$ & PSNR $\uparrow$ & FID $\downarrow$ \\
            \midrule
            Ours & \textbf{ 0.137} & \textbf{30.138} & \textbf{19.45} & \textbf{0.024} & \textbf{34.839} & \textbf{21.19} & \textbf{0.088} & \textbf{25.112} & \textbf{19.25} & \textbf{0.103} & \textbf{28.688} & \textbf{22.62} & \textbf{0.087} & \textbf{29.480} &  \textbf{26.67} \\
            \midrule
            Latent-DPS & 0.324 & 20.086 & 100.27 & 0.249 & 22.64 & 297.43 & 0.227 & 22.184 & 211.23 & 0.390 & 25.608 & 321.5 & 0.950 & -6.753 & 354.95 \\
            PSLD & 0.311 & 20.547 & 42.26 & 0.250 & 22.84 & 214.08 & 0.221 & 22.23 & 204.87 & 0.200 & 23.77 & 318.20 & 0.213 & 23.277 & 359.40 \\
            STSL & 0.614 & 16.063 & 327.38 & 0.476 & 17.859 & 190.64 & 0.436 & 11.843 & 190.64 & 0.583 & 15.196 & 364.07 & 0.604 & 10.095 & 388.68\\
            ReSample & 0.221 & 24.699 & 48.87 & 0.467 & 22.488 & 96.89 & 0.247 & 20.852 & 50.3 & 0.191 & \underline{27.151} & 46.5& 0.281 & \underline{25.138} & 65.06\\
            \bottomrule
            \toprule
            \textbf{FFHQ} & \multicolumn{3}{c}{SR $\times 4$} & \multicolumn{3}{c}{Random Inpainting} & \multicolumn{3}{c}{Box Inpainting} & \multicolumn{3}{c}{Gaussian Deblurring} & \multicolumn{3}{c}{Motion Deblurring}\\
            \cmidrule(r){2-4}\cmidrule(lr){5-7}\cmidrule(l){8-10}\cmidrule(l){11-13}\cmidrule(l){14-16}
            $\sigma_\mathbf{y}=0.1$ & LPIPS $\downarrow$ & PSNR $\uparrow$ & FID $\downarrow$ & LPIPS $\downarrow$ & PSNR $\uparrow$ & FID $\downarrow$ & LPIPS $\downarrow$ & PSNR $\uparrow$ & FID $\downarrow$ & LPIPS $\downarrow$ & PSNR $\uparrow$ & FID $\downarrow$ & LPIPS $\downarrow$ & PSNR $\uparrow$ & FID $\downarrow$ \\
            \midrule
            Ours                 & \textbf{0.1748} & \textbf{24.879} & \textbf{30.107} & \textbf{0.1490} & \textbf{27.536} & \textbf{32.800} & \textbf{0.1631} & \textbf{23.217} & \textbf{26.444} & \textbf{0.1763} & \textbf{25.955} & \textbf{26.083} & \textbf{0.2238} & \textbf{24.612} & \textbf{31.400} \\
            \midrule
            Latent-DPS           & 0.3444 & 19.971 & 45.052 & 0.4455 & 18.117 & 109.83 & 0.6410 & 11.365 & 326.75 & 0.6398 & 13.762 & 330.93 & 0.6360 & 12.524 & 334.43 \\
            PSLD                 & 0.3481 & 19.251 & 47.864 & 0.3105 & 20.588 & 41.737 & 0.3121 & 19.874 & 40.428 & 0.2897 & 21.068 & 36.600 & 0.3307 & 19.224 & 40.374 \\
            STSL                 & 0.3161 & 20.279 & 40.163 & 0.3722 & 19.247 & 54.648 & 0.5481 & 13.864 & 183.00 & 0.5137 & 16.411 & 169.32 & 0.5188 & 15.463 & 163.65 \\
            ReSample             & 0.2613 & 24.184 & 50.224 & 0.5267 & 21.575 & 103.62 & 0.2789 & 20.581 & 53.263 & 0.2984 & 23.980 & 56.489 & 0.6456 & 19.912 & 110.42 \\
            \bottomrule
            \toprule
            \textbf{ImageNet}& \multicolumn{3}{c}{SR $\times 4$} & \multicolumn{3}{c}{Random Inpainting} & \multicolumn{3}{c}{Box Inpainting} & \multicolumn{3}{c}{Gaussian Deblurring} & \multicolumn{3}{c}{Motion Deblurring}\\
            \cmidrule(r){2-4}\cmidrule(lr){5-7}\cmidrule(l){8-10}\cmidrule(l){11-13}\cmidrule(l){14-16}
            $\sigma_\mathbf{y}=0.01$ & LPIPS $\downarrow$ & PSNR $\uparrow$ & FID $\downarrow$ & LPIPS $\downarrow$ & PSNR $\uparrow$ & FID $\downarrow$ & LPIPS $\downarrow$ & PSNR $\uparrow$ & FID $\downarrow$ & LPIPS $\downarrow$ & PSNR $\uparrow$ & FID $\downarrow$ & LPIPS $\downarrow$ & PSNR $\uparrow$ & FID $\downarrow$ \\
            \midrule
            Ours & \textbf{0.238} & \textbf{23.452} & \textbf{39.41} & \textbf{0.142} & \textbf{26.063} & \textbf{34.46} & \textbf{0.230} & \textbf{20.625} & \textbf{37.11} & \textbf{0.253} & \textbf{24.218} & \textbf{38.96} & \textbf{0.203} & \textbf{24.619} & \textbf{38.63}\\
            \midrule
            Latent-DPS & 0.642 & 17.973 & 144.82 & 0.603 & 19.881 & 144.81 & 0.751 & 11.964 & 138.33 & 0.805  & 10.532 & 139.62 & 0.821 & 10.697 & 150.49\\
            PSLD  & 0.380 & 22.690 & 168.08 & 0.306 & 24.167 & 125.25  & 0.330 & 18.290 & 156.30 & \underline{0.397} & 23.076 & 134.18 & \underline{0.453} & \underline{21.576} & 187.21 \\
            STSL &  0.617 & 19.682 & 143.62 & 0.599 & 20.500 & 137.09 & 0.832 & 9.560 & 170.93 & 0.869 & 8.708 & 183.38 &  0.882 & 8.527 &  195.74\\
            ReSample & 0.552 & 20.260 & 133.42 & 0.820 & 17.775 & 229.82 & 0.504 & 16.795 & 138.97 & 0.513 & 21.578 & 116.04 &  0.573 & 20.430 &  145.67\\
            \bottomrule
            \toprule
            \textbf{ImageNet} & \multicolumn{3}{c}{SR $\times 4$} & \multicolumn{3}{c}{Random Inpainting} & \multicolumn{3}{c}{Box Inpainting} & \multicolumn{3}{c}{Gaussian Deblurring} & \multicolumn{3}{c}{Motion Deblurring}\\
            \cmidrule(r){2-4}\cmidrule(lr){5-7}\cmidrule(l){8-10}\cmidrule(l){11-13}\cmidrule(l){14-16}
            $\sigma_\mathbf{y}=0.1$ & LPIPS $\downarrow$ & PSNR $\uparrow$ & FID $\downarrow$ & LPIPS $\downarrow$ & PSNR $\uparrow$ & FID $\downarrow$ & LPIPS $\downarrow$ & PSNR $\uparrow$ & FID $\downarrow$ & LPIPS $\downarrow$ & PSNR $\uparrow$ & FID $\downarrow$ & LPIPS $\downarrow$ & PSNR $\uparrow$ & FID $\downarrow$ \\
            \midrule
            Ours & \textbf{0.4015} & \textbf{22.988} & \textbf{48.211} & \textbf{0.1655} & \textbf{26.043} & \textbf{34.469} & \textbf{0.2428} & \textbf{19.697} & \textbf{46.026} & \textbf{0.4068} & \textbf{22.283} & \textbf{51.131} & \textbf{0.4348} & \textbf{20.428} & \textbf{61.48} \\
            \midrule
            Latent-DPS           & 0.7257 & 15.676 & 147.65 & 0.7973 & 9.4153 & 146.69 & 0.7980 & 9.3345 & 146.51 & 0.7988 & 9.3032 & 193.84 & 0.8525 & 9.1369 & 170.08 \\ 
            PSLD          & 0.4731 & 20.875 & 130.99 & 0.6068 & 19.668 & 145.51 & 0.7028 & 13.909 & 146.74 & 0.7372 & 14.181 & 139.90 & 0.7504 & 13.767 & 149.75 \\ 
            ReSample      & 0.6514 & 18.997 & 155.26 & 0.9654 & 13.612 & 281.82 & 0.5980 & 15.843 & 168.06 & 0.6814 & 19.233 & 173.72 & 1.0461 & 15.249 & 223.52 \\ 
            \bottomrule
        \end{tabular}
    \end{adjustbox}
    
\end{table*}

\subsection{Comparison against other Jacobian-Free Methods (Table \ref{table:main_quantitative_jf})}
\label{sec:jf}
A major advantage of \textbf{DCS} is the fact that it is Jacobian-free (Section \ref{sec:efficiency}) --- this results in at least $6\times$ reduction in memory cost during inference compared with Jacobian-based methods, which can be a major enabling factor for the adoption of such algorithms on consumer GPUs and edge devices. However, naively removing the backpropagation through the score network can reduce the quality of the measurement consistency correction step in inverse solvers. In this experiment, we demonstrate that our treatment via the maximum likelihood framework and the \textbf{n}oise-\textbf{a}ware \textbf{m}aximization results in significantly higher quality samples, compared to a naive implementation in DPS-JF and LGD-JF, which are both Jacobian-free variants of the original algorithms \cite{chung2022diffusion} and \cite{song2023solving}. Namely, we approximate the Jacobian with respect to the input to the denoising network (left hand side) by the Jacobian with respect to the predicted $\mathbf{x}_0$ (right hand side)
\begin{equation}
    \frac{\partial}{\partial \mathbf{x}_t} ||\mathbf{y} - \mathcal{A}(\mathbf{\hat{x}})||_2^2 \approx \frac{\partial}{\partial \hat{\mathbf{x}}} ||\mathbf{y} - \mathcal{A}(\mathbf{\hat{x}})||_2^2,
\end{equation}
where $\hat{\mathbf{x}} = \mathbf{f}(\mathbf{x}_t, \boldsymbol\epsilon_\theta(\mathbf{x}_t, t))$ and $\mathbf{f}$ is an algorithm-dependent function of $\mathbf{x}_t$ and its score. (Note that the right hand side no longer involves backpropagation through $\mathbf{f}$ and therefore $\boldsymbol\epsilon_\theta$).

\begin{table*}
    \caption{Quantitative comparison against other Jacobian-free methods on FFHQ 256x256-1K and ImageNet-1K datasets across various inverse problem tasks and noise levels ($\sigma_{\mathbf{y}} \in \{ 0.01, 0.1 \}$).}
    \label{table:main_quantitative_jf}
    \centering
    \begin{adjustbox}{max width=\linewidth}
        \begin{tabular}{@{}lccccccccccccccc@{}}
            \toprule
            \textbf{FFHQ} & \multicolumn{3}{c}{SR $\times 4$} & \multicolumn{3}{c}{Random Inpainting} & \multicolumn{3}{c}{Box Inpainting} & \multicolumn{3}{c}{Gaussian Deblurring} & \multicolumn{3}{c}{Motion Deblurring}\\
            \cmidrule(r){2-4}\cmidrule(lr){5-7}\cmidrule(l){8-10}\cmidrule(l){11-13}\cmidrule(l){14-16}
            $\sigma_\mathbf{y}=0.01$ & LPIPS $\downarrow$ & PSNR $\uparrow$ & FID $\downarrow$ & LPIPS $\downarrow$ & PSNR $\uparrow$ & FID $\downarrow$ & LPIPS $\downarrow$ & PSNR $\uparrow$ & FID $\downarrow$ & LPIPS $\downarrow$ & PSNR $\uparrow$ & FID $\downarrow$ & LPIPS $\downarrow$ & PSNR $\uparrow$ & FID $\downarrow$ \\
            \midrule
            Ours & \textbf{ 0.137} & \textbf{30.138} & \textbf{19.45} & \textbf{0.024} & \textbf{34.839} & \textbf{21.19} & \textbf{0.088} & \textbf{25.112} & \textbf{19.25} & \textbf{0.103} & \textbf{28.688} & \textbf{22.62} & \textbf{0.087} & \textbf{29.480} &  \textbf{26.67} \\
            \midrule
            DPS-JF & 0.488 & 14.193 & 44.98 & 0.335 & 19.566 & 58.45 & 0.178 & 20.118 & 28.10 & 0.211 & 23.063 & 34.42 & 0.289 & 19.927 & 40.94 \\
            DPS-JF $(T=100)$ & 0.589 & 9.473 & 41.24 & 0.578 & 10.072 & 42.06 & 0.571 & 10.618 & 43.08 & 0.563 & 10.859 & 43.77 & 0.566 & 10.922 & 41.26 \\
            LGD-MC-JF & 0.566 & 10.502 & 41.25 & 0.537 & 12.154 & 43.85 & 0.497 & 13.811 & 46.40 & 0.452 & 15.569 & 46.22 & 0.457 & 15.466 & 46.08 \\
            LGD-MC-JF $(T=100)$ & 0.593 & 9.346 & 40.60 & 0.587 & 9.688 & 40.99 & 0.581 & 10.126 & 42.30 & 0.574 & 10.273 & 40.59 & 0.574 & 10.364 & 40.51 \\
            DDNM & 0.208 & \underline{26.277} & 51.33 & \underline{0.040} & \underline{33.076} & 23.35  &     0.209 &     18.118 & 88.32 & 0.235 & 26.086 & 71.47 & 0.424 & 14.221 & 250.92 \\
            DDRM & 0.502 & 13.002 & 222.45 & 0.393 & 15.935 & 163.91 & 0.472 & 12.148  & 209.18 & - & - & - & - & - & - \\
            \bottomrule
            \toprule
            \textbf{FFHQ} & \multicolumn{3}{c}{SR $\times 4$} & \multicolumn{3}{c}{Random Inpainting} & \multicolumn{3}{c}{Box Inpainting} & \multicolumn{3}{c}{Gaussian Deblurring} & \multicolumn{3}{c}{Motion Deblurring}\\
            \cmidrule(r){2-4}\cmidrule(lr){5-7}\cmidrule(l){8-10}\cmidrule(l){11-13}\cmidrule(l){14-16}
            $\sigma_\mathbf{y}=0.1$ & LPIPS $\downarrow$ & PSNR $\uparrow$ & FID $\downarrow$ & LPIPS $\downarrow$ & PSNR $\uparrow$ & FID $\downarrow$ & LPIPS $\downarrow$ & PSNR $\uparrow$ & FID $\downarrow$ & LPIPS $\downarrow$ & PSNR $\uparrow$ & FID $\downarrow$ & LPIPS $\downarrow$ & PSNR $\uparrow$ & FID $\downarrow$ \\
            \midrule
            Ours                 & \textbf{0.1748} & \textbf{24.879} & \textbf{30.107} & \textbf{0.1490} & \textbf{27.536} & \textbf{32.800} & \underline{0.1631} & \textbf{23.217} & \textbf{26.444} & \textbf{0.1763} & \textbf{25.955} & \textbf{26.083} & \underline{0.2238} & \textbf{24.612} & \textbf{31.400} \\
            \midrule
            DPS-JF & 0.494 & 14.111 & 46.59 & 0.371 & 18.310 & 56.49 & 0.226 & 19.451 & 34.02 & 0.246 & 21.808 & 35.53 & 0.342 & 18.339 & 40.70 \\
            DPS-JF $(T=100)$ & 0.589 & 9.432 & 40.82 & 0.582 & 9.900 & 39.58 & 0.572 & 10.552 & 42.90 & 0.564 & 10.894 & 42.36 & 0.568 & 10.943 & 42.44 \\
            LGD-MC-JF & 0.557 & 11.208 & 44.86 & 0.511 & 13.265 & 49.07 & 0.452 & 15.243 & 48.68 & 0.396 & 17.434 & 46.76 & 0.400 & 17.301 & 45.53 \\
            LGD-MC-JF $(T=100)$ & 0.594 & 9.324 & 41.06 & 0.589 & 9.655 & 41.65 & 0.580 & 10.107 & 42.97 & 0.578 & 10.334 & 41.84 & 0.574 & 10.312 & 41.53 \\
            DDNM                 &      0.6230 &     21.493 &    145.889 &      0.179 &     24.964 &     39.183 &      0.334 &     19.195 &     72.105 &      1.220 &     10.727 &    176.756 &      0.739 &      5.099 &    524.021 \\
            DDRM                 & 0.7853 & 6.3273 & 271.70 & 0.6018 & 10.995 & 255.95 & 0.6323 & 9.6360 & 288.11 & - & - & - & - & - & - \\
            \bottomrule
        \end{tabular}
    \end{adjustbox}
\end{table*}

\subsection{Further Noise Experiments}
\begin{table}[H]
    \caption{Quantitative experiments on FFHQ 256x256-1K at $\sigma_{\mathbf{y}} = 0.5$. We compare against pixel-based solvers (upper half) and latent-based solvers (lower half).}
    \label{table:sigma_0_5_ffhq}
    \centering
    \begin{adjustbox}{max width=\linewidth}
        \begin{tabular}{@{}lccccccccccccccc@{}}
            \toprule
            \textbf{FFHQ} & \multicolumn{3}{c}{SR $\times 4$} & \multicolumn{3}{c}{Random Inpainting} & \multicolumn{3}{c}{Box Inpainting} & \multicolumn{3}{c}{Gaussian Deblurring} & \multicolumn{3}{c}{Motion Deblurring}\\
            \cmidrule(r){2-4}\cmidrule(lr){5-7}\cmidrule(l){8-10}\cmidrule(l){11-13}\cmidrule(l){14-16}
            & LPIPS $\downarrow$ & PSNR $\uparrow$ & FID $\downarrow$ & LPIPS $\downarrow$ & PSNR $\uparrow$ & FID $\downarrow$ & LPIPS $\downarrow$ & PSNR $\uparrow$ & FID $\downarrow$ & LPIPS $\downarrow$ & PSNR $\uparrow$ & FID $\downarrow$ & LPIPS $\downarrow$ & PSNR $\uparrow$ & FID $\downarrow$ \\
            \midrule
            Ours                 & 0.2287 & 20.362 & 100.94 & 0.2067 & 22.999 & 89.312 & 0.2005 & 21.298 & 40.099 & 0.2109 & 24.009 & 82.132 & 0.2301 & 22.306 & 90.403 \\
            \midrule
            DPS                  & 0.2000 & 22.588 & 92.791 & 0.2290 & 22.808 & 90.739 & 0.2118 & 20.278 & 81.491 & 0.2268 & 25.020 & 83.686 & 0.2479 & 20.767 & 91.972 \\
            DDNM                 & 0.7812 & 9.8324 & 387.43 & 0.8721 & 15.573 & 233.15 & 0.9966 & 12.607 & 287.79 & 1.4475 & 3.5686 & 408.85 & 1.3328 & 3.1782 & 393.24 \\
            \midrule
            ReSample             & 0.5704 & 19.948 & 179.35 & 0.6892 & 20.014 & 200.06 & 0.4958 & 17.530 & 160.47 & 0.5409 & 21.166 & 162.40 & 0.6380 & 19.875 & 194.69 \\
            \bottomrule
            \end{tabular}
    \end{adjustbox}
    
\end{table}

\begin{figure}
    \includegraphics[width=\linewidth]{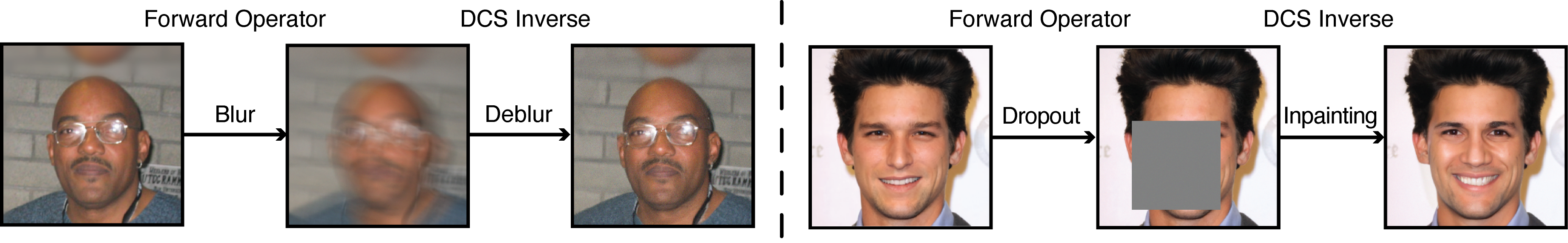}
    \caption{A demonstration of our solver, \textbf{DCS}, solving two inverse problems on natural images from the CelebA-HQ dataset. Motion blur (left), and box dropout (right) are examples of forward operators that are non-invertible. We show further results in Section \ref{sec:experiments}}
    \label{fig:explanatory}
\end{figure}

\subsection{Subset of FFHQ used in other works}

\begin{table}[H]
\caption{Quantitative evaluation of our method on FFHQ 256x256, following the experimental setup of \citep{song2024solving}. We compare against pixel-based solvers (upper half) and latent-based solvers (lower half).}
\label{table:ffhq_subset}
\centering
\begin{adjustbox}{max width=\linewidth}
\begin{tabular}{@{}cccccccccccccccccc@{}}
    \toprule
    & \multicolumn{3}{c}{SR $\times 4$} & \multicolumn{3}{c}{Random Inpainting} & \multicolumn{3}{c}{Box Inpainting} & \multicolumn{3}{c}{Gaussian Deblurring} & \multicolumn{3}{c}{Motion Deblurring} & \multicolumn{2}{c}{Cost} \\
    \cmidrule(r){2-4}\cmidrule(lr){5-7}\cmidrule(l){8-10}\cmidrule(l){11-13}\cmidrule(l){14-16} \cmidrule(lr){17-18} 
    & LPIPS $\downarrow$ & PSNR $\uparrow$ & SSIM $\uparrow$ & LPIPS $\downarrow$ & PSNR $\uparrow$ & SSIM $\uparrow$ & LPIPS $\downarrow$ & PSNR $\uparrow$ & SSIM $\uparrow$ & LPIPS $\downarrow$ & PSNR $\uparrow$ & SSIM $\uparrow$ & LPIPS $\downarrow$ & PSNR $\uparrow$ & SSIM $\uparrow$ & Time $\downarrow$ & Mem. $\downarrow$ \\
    \midrule
    Ours & \textbf{0.074} & \textbf{29.51} & \textbf{0.811} & \textbf{0.052} & \textbf{31.13} & \textbf{0.850} & \textbf{0.102} & \textbf{22.07} & 0.761 & \textbf{0.078} & \textbf{29.92} & \textbf{0.817} & \textbf{0.051} & \textbf{32.32} & \textbf{0.833} & $\mathbf{1x}$ & $\mathbf{1x}$ \\
    \midrule
    DPS & 0.132 & 27.10 & 0.729 & 0.084 & 30.91 & 0.833 & 0.107 & 21.62 & 0.755 & 0.090 & 28.26 & 0.767 & 0.108 & 26.816 & 0.726 & $6\mathbf{x}$ & $3.2\mathbf{x}$ \\
    MCG & 0.112 & 27.07 & 0.784 & 0.877 & 11.02 & 0.02 & 0.905 & 10.883 & 0.001 & 0.176 & 24.89 & 0.768 & - & - & - & $6.1\mathbf{x}$ & $3.2\mathbf{x}$ \\
    DDNM & 0.242 & 27.63 & 0.587 & 0.230 & 27.92 & 0.604 & 0.194 & 23.08 & 0.639 & 0.287 & 27.24 & 0.561 & 0.642 & 8.682 & 0.165 & $1.75\mathbf{x}$ & $1\mathbf{x}$ \\
    \midrule
    Latent-DPS & 0.324 & 20.086 & 0.473 & 0.249 & 22.64 & 0.570 & 0.227 & 22.184 & 0.595 & 0.209 & 23.512 & 0.600 & 0.217 & 22.930 & 0.582 & $6.1\mathbf{x}$ & $8.9\mathbf{x}$ \\
    PSLD & 0.311 & 20.547 & 0.491 & 0.250 & 22.84 & 0.579 & 0.221 & 22.23 & 0.607 & 0.200 & 23.77 & 0.614 & 0.213 & 23.277 & 0.596 & $7.5\mathbf{x}$ & $15\mathbf{x}$\\
    STSL & 0.242 & 27.63 & 0.587 & 0.230 & 27.92 & 0.604 & 0.194 & 23.08 & 0.639 & 0.287 & 27.24 & 0.561 & 0.641 & 10.17 & 0.245 & $1.85\mathbf{x}$ & $9\mathbf{x}$ \\
    ReSample & 0.090 & 29.024 & 0.791 & 0.053 & 30.99 & 0.844 & 0.156 & 20.71 & \textbf{0.778} & 0.113 & 29.19 & 0.784 & 0.197 & 27.65 & 0.706 & $29.5\mathbf{x}$ & $8.95\mathbf{x}$ \\
    \bottomrule
\end{tabular}
\end{adjustbox}
\end{table}

\section{IMPLEMENTATION DETAILS}
\label{sec:implementation_details}
We provide implementation details of our experiments, as well as those for other experiments we compare against.

\subsection{Our Method}
Our proposed \textbf{DCS} has just two primary hyperparameters, as described in the table below. First is the number of time steps $T$. This has relatively little effect on our model performance on most tasks. However, it is occasionally helpful to increase $T$, especially in box inpainting, where there is zero signal from $\mathbf{y}$ in the masked region. Here, higher $T$ allows the diffusion model to obtain a better solution in this unconditional diffusion process. Second, we have the choice of $\texttt{minimizer}$, which is by default the Adam optimizer \cite{kingma2014adam}. However, in the case of linear $\mathcal{A}$, this optimizer can be replaced by the closed form analytical solution to $\mathcal{A}(\mathbf{x}) = \mathbf{y}$. 

For nearly all experiments, we use the Adam optimizer with 50 optimization steps and a learning rate of $1$. The exceptions are the random inpainting and box inpainting tasks, where there is no conditioning information on the masked pixels. This requires more denoising steps, as the diffusion process is totally unconditional inside the mask, up to local correlations learned inside the score network $s_\theta$. Here, we use the analytical solver with $\mathcal{A}^{\dagger} = \mathcal{A}$. Similarly, for nearly all experiments we use $T=50$ as found in Table \ref{fig:robustness}, with the exception being random inpainting and box inpainting tasks, where we found that taking $T=1000$ steps improved performance. However, there is little increase in runtime, since the minimization step is much faster here.

\begin{center}
\begin{tabular}{c | l} 
 \hline
 Notation & \hspace{1.in} Definition\\ [0.5ex] 
 \hline
 $T$ & The number of diffusion steps used in the sampler. \\
 $\texttt{minimizer}$ & The minimizer used to solve for ${\epsilon_\mathbf{y}}$. \\ 
 [1ex] 
 \hline
\end{tabular}
\end{center}

\subsection{Latent Models on ImageNet}
We note that previous latent models use the pretrained weights in \citep{rombach2022high} for $256 \times 256$ resolution datasets. However, there are no published weights in the GitHub repository for unconditional ImageNet, making a fair comparison of our method against latent models more involved. To this end, we leverage a significantly more powerful Stable Diffusion v1.5 model, with publicly available weights on HuggingFace for our experiments. The measurements and the output images are appropriately scaled for a fair comparison.

\subsection{STSL}
At the time of writing this work, we did not find publicly available code for STSL \citep{rout2024solving}. Therefore, we implement the algorithm ourselves in our codebase, and use the hyperparameters provided in the paper.

\subsection{ReSample}
\label{sec:appendix_resample}
We directly use the published code of ReSample  \citep{song2024solving} with no changes in our paper. We discuss two notable aspects of the experiments with ReSample. First, the implementation on GitHub differs from that pseudocode discussed in the paper. Namely, the pseudocode in the paper describes enforcing latent- and pixel-based consistency occasionally during an otherwise unconditional sampling process. 

In the code we observed that the sampling step taken is actually a DPS \citep{chung2022diffusion} sampling step, which includes a posterior-based guidance step that takes an expensive gradient of the noise function. To see this, note that L255 in the \texttt{resample\_sampling} function in \texttt{ddim.py} calls a function \texttt{measurement\_cond\_fn}, which is defined at L62 in \texttt{main.py} and passed into the resampling function. This function is a member of the class \texttt{PosteriorSampling} defined in L53 in 
\texttt{condition\_methods.py}. Inspecting this class, we note that it calls 
\texttt{torch.autograd.grad} on the diffusion step as a function of \texttt{x\_prev} 
(L33 or L39). In other words, a gradient is computed for the measurement norm with respect to the input to the diffusion model, i.e., a DPS step. 

We closely investigated this DPS step in our experiments, ultimately concluding that it has a significant effect on the performance of the algorithm, and that it was a \textit{more} fair comparison to include this step, rather than removing it. However, the inclusion of this sampling step has two primary effects. First, it results in further increases the computation time of ReSample. Second it reveals that ReSample relies significantly on a posterior-based formulation, applying additional resampling steps at each stage.

In experiments, we note that ReSample is significantly slower than other algorithms during sampling (see Table \ref{table:inverse_solvers_overview}). For example, sampling $\sim 1000$ images with ImageNet takes more than two weeks on an A6000 GPU. Since we run five different experimental conditions for each dataset, this was an unacceptably long runtime for our academic resources. Therefore, we reduce the number of diffusion steps $T$ of ReSample in our experiments, from $500$ reported in \citep{song2024solving} to $50$. However, we do provide a single experiment from the \citep{song2024solving} paper, where we reproduce the hyperparameters and dataset (a $100$ image subset of FFHQ). We note that \citep{song2024solving} took a subset of the FFHQ dataset, where performance differed from the full 256$\times$256-1K dataset performance (c.f. Table \ref{table:main_quantitative}). Since the subset was not published, we selected a dataset based where ReSample obtained the same performance with its default parameters in \citep{song2024solving} (Table \ref{table:ffhq_subset}).

\subsection{DDRM}
We used the version of DDRM which is implemented in the DDNM codebase. While DDRM may theoretically be able to handle deblurring tasks, due to the high rank of the forward operators, the SVD cannot be explicitly defined in memory, and no existing code-base for DDRM supplies fast and memory-saving versions of these operators. Because in our settings DDRM's performance was similar or worse than that of DDNM, and due to the fact that DDRM can be considered a subtype of DDNM (see Appendix of \cite{wang2022zero}), we do not run benchmarks.

\section{FURTHER QUALITATIVE COMPARISONS}
\label{sec:qualitative}

We provide further qualitative examples from the FFHQ 256$\times$256-1K and ImageNet 256$\times$256-1K datasets accompanying our quantitative evaluation in Table \ref{table:main_quantitative} on our \href{https://diffusion-conditional-sampling.github.io/}{project website}.

\end{document}